\documentclass[11pt]{article}

\usepackage{graphicx}
\usepackage{amsmath}
\usepackage{xcolor}
\usepackage{hyperref} % clickable toc refs
\usepackage{array}
\usepackage{tabularx,ragged2e,booktabs,caption}
\usepackage[shortlabels]{enumitem}
\usepackage{ulem}
\usepackage{subcaption}
\usepackage{float}
\usepackage{lipsum}
\usepackage{titling}

\usepackage{amsthm}
\usepackage{amssymb}

\usepackage{algorithm}
\usepackage{algpseudocode}

\usepackage{times}
\usepackage[a4paper]{geometry}
\usepackage{amsfonts}

\usepackage{caption}
\usepackage{subcaption}

\usepackage{bm}

\algnewcommand{\LineComment}[1]{\State \(\triangleright\) #1} % https://tex.stackexchange.com/questions/74880/algorithmicx-package-comments-on-a-single-line

\newcommand{\tmix}{t_{\text{mix}}}

\newcommand{\mix}{m}

\newcommand{\cT}{\mathcal{T}}

\newcommand{\tQ}{\tilde{Q}}

\newcommand{\taun}{\tau_{\mix}}

\newcommand{\cZ}{\mathcal{Z}}
\newcommand{\cS}{\mathcal{S}}
\newcommand{\cA}{\mathcal{A}}
\newcommand{\DA}{\Delta_{\vert \mathcal{A} \vert}}

\newcommand{\Qt}{Q^{\pi_t}}
\newcommand{\tQt}{\tilde{Q}^{\pi_t}}

\newcommand{\epsapprox}{\epsilon_{\text{approx}}}
\newcommand{\epsapproxinf}{\epsilon_{\text{app},\infty}}

\newcommand{\ki}[1]{k_{(#1)}}

\newcommand{\unm}{\underline{m}}

\newcommand{\unpi}{\underline \pi}
\newcommand{\cP}{P}

\newcommand{\unu}{\underline{\nu}}
\newcommand{\orho}{\bar{\rho}}
\newcommand{\tunu}{\tilde{\underline{\nu}}}
\newcommand{\torho}{\tilde{\bar{\rho}}}

\newcommand{\Dexpl}{\mathcal{D}_{\text{expl}}}

\newcommand{\tF}{\tilde{F}}
\newcommand{\hF}{\hat{F}}

\newcommand{\oQ}{\bar{Q}}

\newcommand{\tPi}{Z}
\newcommand{\tm}{\tilde{m}}
\newcommand{\tmpi}{\tilde{m}^\pi}
\newcommand{\tpi}{{\tilde{\pi}}}

\newcommand{\epspi}{\varepsilon_\pi}
\newcommand{\epsact}{\varepsilon_\pi}
\newcommand{\epsst}{\varepsilon_\pi}
\newcommand{\piexpl}{\tpi}
\newcommand{\piexplact}{\tpi}
\newcommand{\piexplst}{\tpi}

\newcommand{\sor}{{s}_{or}}
\newcommand{\ukappa}{\underline{\kappa}}

\newcommand{\thetastr}{\bar{\theta}}
\newcommand{\Qstr}{\bar{Q}}
\newcommand{\otheta}{\bar{\theta}}

\newcommand{\oTheta}{\bar{\Theta}}

\usepackage{lineno}
\newcommand{\ob}{\bar{b}}

\newcommand{\LamMax}{\Lambda_{\max}}
\newcommand{\LamMin}{\Lambda_{\min}}
\newcommand{\unLam}{\underline{\Lambda}}

\theoremstyle{plain}
\newtheorem{theorem}{Theorem}[section]
\newtheorem{lemma}[theorem]{Lemma}
\newtheorem{corollary}[theorem]{Corollary}
\newtheorem{proposition}[theorem]{Proposition}

\newtheorem{assumption}[theorem]{Assumption}

\newlength{\leftstackrelawd}
\newlength{\leftstackrelbwd}
\def\leftstackrel#1#2{\settowidth{\leftstackrelawd}%
{${{}^{#1}}$}\settowidth{\leftstackrelbwd}{$#2$}%
\addtolength{\leftstackrelawd}{-\leftstackrelbwd}%
\leavevmode\ifthenelse{\lengthtest{\leftstackrelawd>0pt}}%
{\kern-.5\leftstackrelawd}{}\mathrel{\mathop{#2}\limits^{#1}}}

\let\texdisplaystyle\displaystyle
\def\displaytotextstyle{\textstyle\let\displaystyle\texdisplaystyle}
\newenvironment{talign}
 {\let\displaystyle\displaytotextstyle\align}
 {\endalign}
\newenvironment{talign*}
 {\let\displaystyle\displaytotextstyle\csname align*\endcsname}
 {\endalign}

\usepackage{tikz}

\providecommand{\keywords}[1]{ {\small \textbf{Keywords.} #1}}
\providecommand{\amsclass}[1]{ {\small \textbf{MSC subject classifications.} #1}}

 \usepackage[sort,numbers]{natbib}
\title{Auto-exploration for online reinforcement learning\thanks{Authors listed in alphabetical order.}}
\author{
Caleb Ju\thanks{Email: \texttt{calebju4@gatech.edu}. Partially supported by Department of Energy Computational Science Graduate Fellowship under Award Number DE-SC0022158 and NSF-Cyber-Physical Systems (CPS)-NIFA (2024-67021-43862).} 
\and
Guanghui Lan\thanks{Email: \texttt{george.lan@isye.gatech.edu}. Supported by Office of Naval Research grant (N00014-24-1-2654).}
}
\date{\vspace{-5ex}}

\begin{document}

\maketitle

\begin{abstract}
    The exploration-exploitation dilemma in reinforcement learning (RL) is a fundamental challenge to efficient RL algorithms. Existing algorithms for finite state and action discounted RL problems address this by assuming sufficient exploration over both state and action spaces. However, this yields non-implementable algorithms and sub-optimal performance. To resolve these limitations, we introduce a new class of methods with auto-exploration, or methods that automatically explore both state and action spaces. 
    Auto-exploration can be applied in both the tabular and linear function approximation settings. 
    Under algorithm-independent assumptions on the existence of an exploring optimal policy, both settings attain $O(\epsilon^{-2})$ sample complexity to solve to $\epsilon$ error. These complexities are novel since they avoid algorithm-dependent parameters seen in prior works, which may be arbitrarily large.  The methods are also simple to implement because they are parameter-free.
    We achieve these results by integrating auto-exploration into policy mirror descent to avoid the (unknown) stationary distribution seen in prior art. 
    In the tabular setting, we introduce a dynamic exploration time with a data-driven stopping time, while for linear function approximation, we propose a new sampling distribution based on the discounted visitation distribution that covers a more general class of Markov chains.
\end{abstract}


\vspace{1em}
\amsclass{90C15, 90C30, 90C40 60J10, 68T05}

\section{Introduction}
Markov decision processes (MDPs) are a fundamental mathematical framework for decision-making under an uncertain, dynamic environment, which have been intensively studied since the 1950s~\cite{bellman1958dynamic,puterman2014markov}. 
In this paper, we consider infinite-horizon, finite state and action, discounted MDPs, defined by the five-tuple $(\cS, \cA, \cP, c, \gamma)$, with states $\cS$, actions $\cA$, transition probability matrix $\cP$, costs $c$, and discount factor $\gamma \in (0,1)$.
The goal is to find an optimal policy, which provides the optimal action (possibly in a randomized fashion) at any given state in terms of minimizing the cumulative discounted cost.
A precise definition is provided later in our background section.

Recently, reinforcement learning (RL) has gained significant attention due to its empirical success in AI strategic gameplay, training large language models, controlling complex processes, and robotics~\cite{khan2020systematic,enda13applying,mnih2015human,ouyang2022training}. 
RL can be viewed as the data-driven analogue of MDPs, where the dynamics defined by $\cP$ are unknown and must be inferred from data in the form of a sequence of state-action-cost triplets.
The three most common data generation frameworks are offline, generative (or simulator), and online (or single trajectory)~\cite{sutton1998reinforcement}. 
The offline approach assumes a batch of data has already been collected~\cite{wang2020statistical,kallus2022efficiently,jiang2016doubly}, while the generative approach assumes access to a simulator that can arbitrarily generate new independent and identically distributed (iid) data~\cite{gheshlaghi2013minimax,agarwal2021theory,li2024breaking}. 
However, one may not have access to such convenient datasets.
On the other hand, the online approach uses a single stream of data, which we assume is generated by a Markov chain.
While the online approach is the most general since it can handle both iid~and Markovian data (i.e.,~sequential data satisfying the Markov property), it is the most challenging due to the presence of Markovian noise~\cite{kotsalis2022simple}. 
Online models cannot directly choose states to sample from, like in the generative model.
Instead, they can only influence it through action selection. 
Furthermore, this paper focuses on the \textit{on-policy} case, where actions are sampled with the current, learned policy.
Although on-policy sampling exploits past information by selecting ``good'' actions with higher probability, other actions are visited less frequently or never selected.
Altogether, this leads to the exploration issue, where there is insufficient data for hard-to-reach states and actions. 

To avoid poor exploration, two common strategies are \textit{explicit (action) exploration} ($\epsilon$-greedy exploration~\cite{li2025policy,mnih2015human,nanda2025minimal}) and function approximation. The former selects an action uniformly at random from time to time to ensure every action is visited with probability greater than zero. 
The latter utilizes linear and nonlinear function approximation techniques, such as neural networks, to find low-dimensional features~\cite{mnih2015human,fan2020theoretical}. 
While both tools are indispensable in empirical RL, there are still fundamental gaps in terms of sample complexity and convergence guarantees. 
For example, existing methods commonly assume sufficient action exploration, or that any action can be selected with probability greater than zero~\cite{lan2023policy,li2023accelerated,mou2023optimal,chen2022finite,kotsalis2022simple,zhu2024uncertainty}. 
Unfortunately, this assumption may not hold for on-policy sampling methods in unregularized finite state and action RL problems, since the policy may converge to a deterministic optimal policy~\cite{puterman2014markov}. 
A natural remedy is explicit exploration. 
However, directly applying explicit exploration in policy-gradient methods yields a sub-optimal $O(\epsilon^{-4})$ sample complexity~\cite{alacaoglu2022natural}, where $\epsilon > 0$ is the accuracy tolerance. This can be improved to $O(\epsilon^{-2})$ with a carefully timed and controlled explicit exploration schedule~\cite{li2024stochastic}, which also holds under linear function approximation.
While this may suggest explicit exploration is necessary for efficient RL,~\cite{li2025policy} showed this is not the case. They introduced a new inherent (action) exploration phenomenon, which we will refer to as \textit{implicit exploration} to contrast with explicit exploration. 
This phenomenon ensures that under the existence of a deterministic optimal policy, the sample complexity is $O(\epsilon^{-2})$ without explicit exploration. 
However, in contrast to~\cite{li2024stochastic}, this result only holds for tabular RL and not for function approximation.

Unlike action exploration, state exploration is much less well understood in the online setting. 
To quantify the hardness of state exploration, we define a \textit{difficulty of exploration} parameter $\Dexpl$, which quantifies the cost of sufficient state exploration. 
This parameter goes under different names and definitions depending on the properties and analysis of the RL algorithm. 
For example, it can be viewed as the distribution-mismatch~\cite{agarwal2021theory}, transfer error~\cite{agarwal2020pc}, concentrability coefficient~\cite{munos2005error}, covering time~\cite{even2003learning}, diameter~\cite{tarbouriech2021provably}, or the mixing time of a Markov chain~\cite{lan2023policy,li2024stochastic,chen2022finite}. 
Because $\Dexpl$ often depends on the behavior of a stochastic RL algorithm, it is in general an \textit{algorithm-dependent} parameter. 
Equivalently, it lacks \textit{algorithm-independent} bounds, which only depend on the problem and not on an algorithm's behavior. 
For instance,~\cite{li2024stochastic,li2025policy} established $O(\Dexpl \cdot \epsilon^{-2})$ sample complexity for policy-gradient methods, where $\Dexpl$ is related to the mixing rate and stationary distribution of every Markov chain induced by an intermediate policy. 
As a result, $\Dexpl$ can be arbitrarily large.
It was suggested that $\Dexpl$ can be algorithm-independent through implicit exploration, but no rigorous bounds were provided~\cite[Remark 5.2]{li2025policy}.
It was later shown that under a setup similar to~\cite[Remark 5.2]{li2025policy}, a Q-learning-type method attains $O(\Dexpl \cdot \epsilon^{-2})$ sample complexity~\cite{nanda2025minimal}. 
Here, $\Dexpl = O(\lambda^{-\ob})$, where $\lambda$ is the minimum action probability for any intermediate policy, and $\ob \geq 2$ is an algorithm-independent constant.
Under $\epsilon$-greedy exploration, then $\lambda \geq \epsilon/\vert \cA \vert$, yielding an algorithm-independent but sub-optimal $O(\epsilon^{-(2+\ob)})$ sample complexity. 
Therefore, an algorithm-independent $O(\epsilon^{-2})$ sample complexity remains open for both the tabular (with on-policy sampling) and function approximation settings.

Another major issue with existing online methods for RL is their practical implementation. To select hyperparameters like the number of samples to collect, exploration strength, stepsize (learning rate), and iteration count, both explicit and implicit exploration require knowledge of the (unknown) mixing rate and stationary distribution of a Markov chain~\cite{nanda2025minimal,li2025policy,kotsalis2022simple}. 
While these works assume that such parameters can be estimated, current estimation procedures lack non-asymptotic guarantees when the quantities of interest are unknown~\cite{wolfer2019estimating}. 
There is some progress on developing parameter-free (or prior-free) methods for RL, which do not need prior knowledge of problem-dependent parameters, such as the mixing rate~\cite{lee2025near,tarbouriech2021provably,tuynman2024finding}. 
However, these results are either limited to the generative model or have algorithm-dependent bounds in the online model.
As a result, the hyperparameters must be tuned heuristically before running the algorithm, which lacks convergence guarantees and introduces overhead when implementing these RL methods.

\subsection{Contributions and paper outline}
We introduce a new class of stochastic policy mirror descent (SPMD) methods~\cite{lan2023policy} with automatic exploration, or \textit{auto-exploration} for short, for solving finite state and action, discounted RL problems in the online model. 
Our main result shows that SPMD with auto-exploration in the tabular setting achieves $\epsilon$-accuracy with probability $1-\delta$ using at most
\begin{talign*} % \label{eq:auto_complexity}
    \tilde{O}(\Dexpl(\delta) \cdot \vert \cA \vert^{2- \frac{1}{2 + \log_2(1/[1-\gamma])}}(1-\gamma)^{-5}\epsilon^{-2})
\end{talign*}
samples, assuming that there exists an optimal policy that is sufficiently exploring.
Here, $\tilde{O}(\cdot)$ only hides either absolute, polylogarithmic, or algorithm-independent terms.
The constant $\Dexpl(\delta)$ only has polylogarithmic dependence on $1/\delta$, polynomial dependence on problem-specific parameters, and no dependence on $\epsilon$ or other algorithm-dependent terms. 
See Theorem~\ref{thm:implicit_convergence} for more details.
This appears to be the first time an algorithm-independent, on-policy $\tilde{O}(\epsilon^{-2})$ sample complexity is achieved under weaker exploration assumptions.
In addition, we also establish a similar sample complexity for an auto-explorative SPMD with linear function approximation.

Our techniques to achieve auto-exploration rest upon new developments in both the tabular and function approximation settings.
In the tabular setting, we introduce a \textit{dynamic exploration time} to make SPMD parameter-free.
The dynamic exploration time is easily accessible in practice and replaces the role of the (unknown) mixing time appearing in prior art.
We show that the dynamic exploration time not only shares the same upper bound as the mixing time, but it is also less conservative in practice.
Then, in the function approximation setting with linear basis functions, we propose a new sampling distribution for policy evaluation based on the discounted visitation distribution, which may be of independent interest. 
Unlike the (unknown) stationary distribution used in prior art, this distribution always exists and does not require estimating the (unknown) mixing rate.
Third, we establish a high-probability, last-iterate convergence of SPMD with improved sample complexity. 
Our new result improves upon prior last-iterate sample complexity by a factor of $O\{(1-\gamma)^{-2}\}$~\cite{lan2023policy}, as well as upon average-iterate by a factor of $O\{\vert \cA \vert^{1/2}/(1-\gamma)\}$~\cite{li2025policy}.
These improvements are enabled through a new choice and analysis of optimization parameters in SPMD that adapt with respect to (w.r.t.) the discount factor.
Fourth, we provide numerical experiments that show our methods perform favorably compared to a popular RL baseline and methods that estimate the (unknown) mixing rate across several environments, including a route planning, inventory, and classical robotics control problem.
In particular, our dynamic approach observes 34x to 485x fewer samples per policy evaluation step. 
Empirically, the dynamic approach with linear function approximation can improve the policy in environments without sufficient state exploration, suggesting that these methods can potentially handle more general environments.

This paper is organized as follows.
In Section~\ref{sec:basics}, we introduce the basics of MDPs and RL, as well as SPMD.
Then in Section~\ref{sec:implicit}, we introduce and analyze the first SPMD with auto-exploration for the tabular setting.
Next in Section~\ref{sec:spmd_linear}, we introduce and analyze SPMD using function approximation with auto-exploration.
We finish with experiments in Section~\ref{sec:experiments} and the conclusion in Section~\ref{sec:conclusion}. 
To close this section, we will discuss related works.

\subsection{Other related works}

\noindent \textit{Last-iterate convergence of RL methods}.
Much of the existing convergence guarantees for RL are w.r.t.~regret (i.e.~average-iterate convergence) or best-iterate~\cite{li2025policy}. Note that the best-iterate is not easy to compute in the data-driven setting. 
Recently, expected last-iterate convergence has been shown for policy-gradient and Q-learning~\cite{ju2024strongly,chen2025non}, but no high-probability results were provided. 
High-probability, last-iterate convergence has been proven for approximate dynamic programming, where one estimates the action-value function to high-accuracy and then applies the classical value or policy iteration~\cite{munos2005error,kakade2002approximately}. 
While this approach delivers efficient sample complexity, the results will have poor worst-case dependence on the discount factor in the online setting.
Other works have examined last-iterate convergence in the more challenging constrained or multi-agent setting~\cite{ding2023last,cen2022faster}, but they either require a generative model or the true model. \newline

\noindent \textit{Exploration in episodic RL}. 
Another popular data generation model for RL is the so-called episodic setting. Here, the data-generating process can be reset to a fixed initial state or distribution after a set period of time. Compared to the generative model, the episodic setting is less flexible because one cannot arbitrarily choose the starting point. On the other hand, it is stronger than the general online model since the data can be arbitrarily reset to obtain iid~samples. In this setting, sublinear regret bounds have been developed~\cite{cai2020provably,jin2018q} to solve finite-horizon RL problems. Their sample complexity critically avoids exploration constants like $\Dexpl$ by applying bonus terms to visit under-explored actions and resetting the state to a target distribution. Such efficient results, compared to the general online setting, further cement the challenge in solving RL problems with an online model. 

\subsection{Notation}
Throughout this paper, we denote all logarithms in base 2 unless otherwise explicitly written.
We denote the probability simplex over $n$ elements as $\Delta_n := \{ x \in \mathbb{R}^{n} : \sum_{i=1}^n x_i = 1, x_i \geq 0 \ \forall i \}$.
Throughout this paper, we assume the norm satisfies $\|x\| \leq 1$ for any vector $x \in \DA$ that is a probability distribution. 
Clearly, this holds for all $\ell_p$ norms where $p \geq 1$. For any diagonal matrix $W$ and vector $x$ of the same dimension, we write the weighted norm $\|x\|_W := \sqrt{x^TWx}$.

A widely used tool in the optimization literature is the Bregman divergence (or prox-function). For a distance generating function $\omega : X \to \mathbb{R}$ for some set $X \subseteq \mathbb{R}^n$, it is defined as 
$D(q,p) := \omega(p) - \omega(q) - \langle \nabla \omega (q), p-q \rangle, \ \forall p,q \in X$.
The choice of $X = \Delta_n$ and Shannon entropy $\omega(p) := \sum_{i=1}^n p_i \log p_i$ results in a Bregman divergence as the KL-divergence. 
We assume without loss of generality (by scaling) that $\omega$ is strongly convex with modulus $\mu_\omega = 1$ w.r.t.~a user-chosen norm $\|\cdot\|$ (we will verify this later). 

\section{New convergence guarantees for stochastic policy mirror descent} \label{sec:basics}
We start by reviewing the basics of Markov decision processes (MDPs) before discussing our main algorithmic workhorse, stochastic policy mirror descent. 

\subsection{Background on MDPs} \label{sec:background}
An infinite-horizon, discounted Markov decision process (MDP) is a five-tuple $(\cS, \cA, \cP, c, \gamma)$, where $\cS$ is the state space, $\cA$ is the action space, and $\cP : \cS \times \cS \times \cA \to \mathbb{R}$ is the transition probability matrix (for any state-action $(s,a)$, $\cP(s' \vert s,a)$ is the probability of the next state being $s'$). 
We assume $\cS$ and $\cA$ are both finite sets. 
The cost is $c : \cS \times \cA \to \mathbb{R}$, and $\gamma \in [0,1)$ is a discount factor. 
A feasible, randomized policy $\pi : \cA \times \cS \to \mathbb{R}$ determines the probability of selecting a particular action at a given state. We denote the space of feasible, randomized policies by $\Pi$.
We measure a policy $\pi$'s performance by the action-value function $Q^\pi : \cS \times \cA \to \mathbb{R}$, defined as
\begin{talign*}
    Q^\pi(s,a) :=
    &~\mathbb{E}_{\substack{a_t \sim \pi(\cdot \vert s_t) \\ s_{t+1} \sim \cP(\cdot | s_t, a_t)}} [\sum_{t=0}^\infty \gamma^t \big [c(s_t, a_t) 
    \\& 
    + h^{\pi(\cdot \vert s_t)}(s_t)] \mid
    s_0 = s, a_0 = a \big],
\end{talign*}
where the function $h^\cdot(s) : \DA \to \mathbb{R}$ is a closed, strongly convex function with modulus $\mu_h \geq 0$ w.r.t.~the policy $\pi(\cdot \vert s)$, or that
\begin{talign*}
    h^{\pi(\cdot \vert s)}(s) - \ell_{h}(\pi', \pi; s)
    \geq
    \mu_h D^\pi_{\pi'}(s), 
\end{talign*}
where $\ell_h(\pi',\pi; s) := h^{\pi'(\cdot \vert s)}(s) + \langle (h')^{\pi'(\cdot \vert s)}(s, \cdot), \pi(\cdot \vert s) - \pi'(\cdot \vert s) \rangle$ denotes the linear approximation of $h^\cdot(s)$ at $\pi'(\cdot \vert s)$, $\langle \cdot, \cdot \rangle$ denotes the inner product over $\mathbb{R}^{\vert \cA \vert}$, $(h')^{\pi'(\cdot \vert s)}(s, \cdot)$ denotes a subgradient of $h^\cdot(s)$ at $\pi'(\cdot \vert s)$, and the Bregman divergence between any two policies at state $s$ is
\begin{talign*}
    D_{\pi}^{\pi'}(s)
    :=
    D(\pi(\cdot \vert s), \pi'(\cdot \vert s)).
    % =
    % \omega(\pi'(\cdot \vert s)) - \omega(\pi(\cdot \vert s)) - \langle \nabla \omega (\pi(\cdot \vert s)), \pi'(\cdot \vert s) - \pi(\cdot \vert s) \rangle.
\end{talign*}
We assume there is an $M_h$ such that $h^\cdot(s)$ is $M_h$-Lipschitz for all $s$, or that for any $\pi',\pi \in \Pi$,
\begin{talign*}
    h^{\pi'(\cdot \vert s)}(s) - h^{\pi(\cdot \vert s)}(s)
    \leq
    M_h\|\pi'(\cdot \vert s) - \pi(\cdot \vert s)\|.
\end{talign*}
The regularization $h^{\pi(\cdot \vert s)}$ can model the popular entropy regularization to induce safe exploration and learn risk-sensitive policies, as well as barrier functions for constrained MDPs.
While $h^{\pi(\cdot \vert s)}$ is separated from the cost $c(s,a)$ to achieve faster convergence when $\mu_h > 0$, this paper considers the simpler $\mu_h \geq 0$. 
We refer to~\cite{lan2023policy} for handling $\mu_h > 0$. 
We assume all costs are bounded, which, without loss of generality, can be stated as $c(s,a) + h^{\pi(\cdot \vert s)}(s) \in [0,1]$. 
We also define the state-value function $V^\pi : \cS \to \mathbb{R}$ w.r.t.~$\pi$ to be
\begin{talign}\label{eq:def_V_function}
    V^\pi(s) :=
    &~\mathbb{E}_{\substack{a_t \sim \pi(\cdot \vert s_t) \\ s_{t+1} \sim \cP(\cdot | s_t, a_t)}} [\sum_{t=0}^\infty \gamma^t \big [c(s_t, a_t) 
    \\& 
    + h^{\pi(\cdot \vert s_t)}(s_t)] \mid
    s_0 = s \big], \nonumber
\end{talign}
It can be easily seen from the definitions of $Q^\pi$ and $V^\pi$ that
$V^\pi(s)
=
\sum_{a \in \cA} \pi(a \vert s) Q^{\pi}(s,a)
=
\langle Q^\pi(s, \cdot), \pi(\cdot \vert s) \rangle$
and
$Q^\pi(s, a) = c(s, a) +h^{\pi(\cdot \vert s)}(s) + \gamma \sum_{s' \in \cS}  \cP(s'| s, a) V^\pi(s')$.

The main objective in MDPs is to find an optimal policy $\pi^* : \cA \times \cS \to \mathbb{R}$ satisfying the value optimality condition: $V^{\pi^*}(s) \le V^\pi(s), \forall \pi(\cdot \vert s) \in \DA, \forall s \in \cS$.
Sufficient conditions that guarantee the existence of $\pi^*$ have been intensively studied (e.g.,~\cite{bertsekas1996stochastic,puterman2014markov}).
Note that the value optimality condition can be formulated as a nonlinear optimization  problem with a single objective function.
Given an initial state distribution $\rho \in \Delta_{\vert \cS \vert}$, one solves
\begin{talign*}
\min_{\pi \in \Pi} \{f_\rho(\pi) := \sum_{s \in \cS} \rho(s) \cdot V^\pi(s)\}.
\end{talign*}
When $\rho$ is strictly positive, the optimal solution satisfies the value optimality condition.
While prior policy-gradient methods typically fix $\rho$ to be the stationary state distribution $\nu^* := \nu^{\pi^*}$ induced by the optimal policy $\pi^*$, we aim for the stronger \textit{distribution-free} convergence, where convergence is measured w.r.t.~any distribution $\rho$~\cite{ju2024strongly}. 

Next, we briefly discuss the structural properties of MDPs and RL that will provide convergence guarantees.

\subsection{Performance difference and advantage function} \label{sec:perf_diff_and_gap}
Given a policy $\pi(\cdot \vert s) \in \DA$, we define the discounted state visitation distribution $\kappa^\pi_q : \cS \to \mathbb{R}$ by
\begin{talign}
    \kappa^\pi_q(s)
    &:=
    (1-\gamma) \sum_{t=0}^\infty \gamma^t P_\pi^{(t)}(q,s) \label{eq:visitation_measure},
\end{talign}
where $P_\pi^{(t)}(q,s) := \mathrm{Pr}^{\pi}\{s_t = \cdot \vert s_0=q\}$ is the distribution of state $s_t$ when following policy $\pi$ and starting at state $q \in \cS$.
In the finite state case, we can also view $\kappa^\pi_q \in \mathbb{R}^{\vert \cS \vert}$ as a vector.

We now state an important ``performance difference'' lemma
which tells us the difference in the value functions for two policies.
A proof can be found in~\cite[Lemma 1]{ju2022policy}.
\begin{lemma} \label{lem:performance_diff_deter}
    Let $\pi$ and $\pi'$ be two feasible policies. Then we have for any $s \in \cS$,
    \begin{talign*}
        V^{\pi'}(s) - V^\pi(s)
        &= \frac{1}{1-\gamma} \sum_{q \in \cS}  \psi^\pi(q, \pi'(\cdot \vert q)) \kappa_s^{\pi'} (q), 
    \end{talign*}
    where we denote
        $\psi^\pi(s, p)
        := \langle Q^\pi(s, \cdot), p \rangle -V^\pi(s) +  h^{p}(s) - h^{\pi(\cdot \vert s)}(s)$
    as the advantage function for a given $p \in \DA$.
\end{lemma}
Using these structural properties, one can show that the following policy-gradient type method solves RL problems to global optimality.

\subsection{Stochastic policy mirror descent}
Recall that the stochastic policy mirror descent~\cite{lan2023policy} (SPMD) is an iterative algorithm, where at iteration $t$ it performs the simple update across all states $s \in \cS$,
\begin{align} \label{eq:spmd_update}
    \pi_{t+1}(\cdot \vert s) 
    =
    \mathrm{argmin}_{\upsilon \in \DA} \big\{ &\langle \tilde{Q}^{\pi_t}(s,\cdot), \upsilon \rangle + h^{\upsilon}(s) 
    \\ & 
    + \eta^{-1}_t D(\pi_t(s), \upsilon)\big\}, \nonumber
\end{align}
where $\tilde{Q}^{\pi_t}$ is a stochastic estimate of $Q^{\pi_t}$, $\eta_t$ is the stepsize, and $D(\cdot, \cdot)$ is the Bregman divergence.
A closed-form solution exists for certain $D(\cdot,\cdot)$ and $h$, such as $D(q,p) = \mathrm{KL}(p || q)$ and $h^\upsilon(s) = \sum_{a \in \cA} \upsilon_a \log(\upsilon_a)$.  
Otherwise, one can efficiently solve the sub-problem to high accuracy. 
See~\cite{lan2023policy} for more details. 
Later in subsection~\ref{sec:yanli_work}, we will choose a Bregman divergence induced by a negative Tsallis divergence, and the sub-problem can be efficiently solved with bisection search~\cite{li2025policy}. 
We next look to establish new convergence guarantees for SPMD.

\subsection{New high-probability, last-iterate convergence for SPMD}
In what follows, we require only two assumptions, which will be verified later in the paper. 
First,~\eqref{eq:second_moment} requires almost sure boundedness of the action-value function. 
Second,~\eqref{eq:weight_diff_bias} asks for the weighted bias to be small, similar to~\cite[Lemma 5.2, Theorem 5.1]{li2025policy}.
We define $\xi_t$ to be the random vector used to generate the stochastic estimator $\tilde{Q}^{\pi_t}$. 
Denote the collection of random vectors $\xi_{[t-1]} := \{\xi_0,\xi_1,\ldots,\xi_{t-1}\}$ as the complete history up to time $t-1$ (inclusively). 
\begin{proposition} \label{prop:stronger_gap_converge_with_relaxed}
    For any $k \geq 4$ and $\delta \in (0,1)$, choose a stepsize $\eta_t = \alpha/\sqrt{k}$.
    In addition, suppose for $t = 0,\ldots,k-1$ we have almost surely at every $s \in \cS$,
    \begin{talign}
        \|Q^{\pi_t}\|_*^2, \|\hat{Q}^{\pi_t}\|_*^2 &\leq \hat{Q}^2 \label{eq:second_moment}  \\
        \mathbb{E}_{\xi_{t} \vert \xi_{[t-1]}} [\langle -\delta_t(s,\cdot), \pi_t(\cdot \vert s) - \pi^*(\cdot \vert s) \rangle] &\leq \varsigma, \label{eq:weight_diff_bias}
    \end{talign}
    for some $\hat{Q},\varsigma \geq 0$,  where $\delta_t(s,a) := \tilde{Q}^{\pi_t}(s,a) - \Qt(s,a)$.
    Then
    \begin{talign}
        \Pr\{ \exists s \in \cS &: V^{\pi_k}(s) - V^{\pi^*}(s) 
        >
        E_1(k,\varsigma,\delta)
         \}
        \leq 
        \delta, \label{eq:last_iter_opt_gap}
    \end{talign}
    where $E_1(k,\varsigma,\delta) := \frac{4\alpha^{-1}[\bar{D}_0+1] + 67\alpha \cdot (\hat{Q}^2 + M_h^2)(\log \frac{4k\vert \cS \vert}{\delta})^2}{(1-\gamma)\sqrt{k}} + \frac{18\log(4k)\varsigma}{1-\gamma}$ and $\bar{D}_0$ satisfies $\max_{s \in \cS} D^{\pi^*}_{\pi_0}(s) \leq \bar{D}_0$.
    If instead~\eqref{eq:second_moment} and~\eqref{eq:weight_diff_bias} hold for $t=0,\ldots,t'-1$, where $t' \leq k$, then
    \begin{talign} \label{eq:dist_to_opt_at_t}
      \mathrm{Pr}\big\{ \exists s \in \cS, t \leq k-1 ~:~ D^{\pi^*}_{\pi_{t}}(s)
      >
      E_2(k,\varsigma,\delta) 
      \big\}
      \leq \delta,
    \end{talign}
    where $E_2(k,\varsigma,\delta) := \frac{\bar{D}_0 + \alpha^2 (\hat{Q}^2 + M_h^2) + 4\alpha \hat{Q}\sqrt{\log(k \vert \cS \vert/\delta)} + 2\alpha\varsigma\sqrt{k}}{1-\gamma}$.
\end{proposition}
Proof of~\eqref{eq:dist_to_opt_at_t} is identical to~\cite[Lemma 5.5]{li2025policy}.
The last-iterate convergence~\eqref{eq:last_iter_opt_gap} under a bounded, low-bias estimator (\eqref{eq:second_moment}~and~\eqref{eq:weight_diff_bias}, respectively) appears to be new, and it will help improve the state-of-the-art sample complexity when paired with a Monte-Carlo estimator later.
The main technical difficulty in deriving~\eqref{eq:last_iter_opt_gap} is that the average deviation error between two consecutive policies can not be made small by Bernstein-like inequalities. 
Rather, we use Lemma~\ref{lem:performance_diff_deter} with the analysis in~\cite{jain2021making} to control the deviation error between groups of iterates.
We defer the full details of the proof to~\ref{sec:last_iterate}.

In the next section, we show how to satisfy conditions~\eqref{eq:second_moment} and~\eqref{eq:weight_diff_bias}.

\section{Auto-exploration for tabular RL} \label{sec:implicit}
This section develops a parameter-free implementation of SPMD in the tabular setting. We start by reviewing a simple Monte Carlo estimator.

\subsection{Review of a simple Monte Carlo estimator} \label{sec:tomc_review}
In this section and in policy evaluation for function approximation (Section~\ref{sec:spmd_linear}), we condense the regularized cost $c(s,a) + h^{\pi(\cdot \vert s)}(s)$ into $c(s,a)$ for simplicity, similar to~\cite{li2024stochastic}. 
This can be done since $\pi$ is a fixed policy during policy evaluation. 
We denote the state-action space, $\cZ := \cS \times \cA$.

First, we define the truncated on-policy Monte Carlo (TOMC) estimator~\cite{li2025policy}.
At each iteration of SPMD, we collect samples $z_0,z_1,\ldots$, where $z_t := (s_t,a_t)$ (we can always re-index the time index after each SPMD iteration).
For an exploration time $m \geq 0$, define the set $\cT_{m}(z) := \{t \in [0,m) : z_t=z\}$. When $\cT_m(z) \ne \varnothing$, let
\begin{talign} \label{eq:hitting_time}
    \taun^\pi(z) := \min \cT_{m}(z).
\end{talign}
Otherwise, when $\cT_m(z) = \varnothing$, let $\taun^\pi(z) = m$.
The random set $\cT_m(z)$ contains all the times the state-action $z$ is encountered between time $0$ and time $t-1$. 
Meanwhile, $\taun^\pi(z)$ is the first hitting time of state $z$, and states not visited within the exploration time are assigned values $m$. 
We also define the space of ``non-rare'' state-action pairs w.r.t.~the probability provided by~$\pi$,
\begin{talign} \label{eq:tilde_Pi_t_defn}
    \tPi^{\pi}(\unpi) := \{
        (s,a) : \pi(a \vert s) \geq \unpi
    \},
\end{talign}
for some to-be-determined constant $\unpi > 0$.
Later, we will show that $\tPi^{\pi}(\cdot)$ contains optimal actions with high probability for a properly chosen $\unpi$.
With these quantities, we define the TOMC estimator for the action-value function (the sum over an empty set of indices is 0) for a $z \in Z^\pi(\unpi)$ as,
\begin{talign} \label{eq:montecarlo_Q}
    \tQ_{m,\unpi}(z) := 
    \sum_{t=\taun^\pi(z)}^{m-1} \gamma^{t-\taun^\pi(z)} \cdot c(z_t).
\end{talign}
That is, $\tQ_{m,\unpi}(z)$ is a simple Monte-Carlo estimate.
Otherwise, when $z \not \in \tPi^{\pi}(\unpi)$, the estimator is set to a large value $\tQ_{m,\unpi}(z) = (1-\gamma)^{-1}$ to de-incentivize future policies from selecting $z$.
We next look to derive convenient accuracy estimates for $\tQ_{m,\unpi}$.

\subsection{Parameter-free Monte Carlo estimator} \label{sec:param_free_tomc}
We will now show $\tau^\pi_m(z)$ from~\eqref{eq:hitting_time} provides a convenient bound on the bias of $\tQ_{m,\unpi}(z)$.
Recall $P^{(t)}_\pi(\cdot,\cdot)$ defined below~\eqref{eq:visitation_measure}. 
We denote the one-step transition, $P_{\pi} = P_{\pi}^{(1)}$, for a policy $\pi$.
\begin{lemma} \label{lem:q_bias}
    For any $z \in \tPi^\pi(\unpi)$, the bias is bounded by
    \begin{talign*}
        \big \vert \mathbb{E}[ \tQ_{m,\unpi}(z) - Q^{\pi}(z) \vert \tau_{m}^\pi(z)] \big \vert
        &\leq
        (1-\gamma)^{-1} \gamma^{m - \tau^\pi_m(z)}.
    \end{talign*}
    Moreover, suppose the Markov chain defined by $P_\pi$ mixes towards a unique stationary state distribution $\nu \in \Delta_{\vert \cS \vert}$, where
    \begin{talign} \label{eq:l1_mixing}
        \|P_\pi^{(\tau)}( \cdot \vert s_0) - \nu \|_{1}
        \leq
        C \cdot \rho^{\tau+1}, \ \forall s_0 \in \cS,~\forall \tau > 0,
    \end{talign}
    for some parameters $C > 0$ and $\rho \in [0,1]$. 
    Then for any fixed $z \in Z$, we have
    \begin{talign} \label{eq:hitting_time_bound}
        \Pr\{ \gamma^{m - \tau^\pi_m(z)} > 3(1-\beta_\delta^\pi(z))^{m}\}
        &\leq
        \delta,
    \end{talign}
    where $\beta_\delta^\pi(z) 
    := 
    \min\{ 1-\sqrt{\gamma}, 
    \frac{\pi(a \vert s) \nu(s) }{6\tmix^\pi(z) \log(\frac{4\tmix^\pi(z)}{\pi(a \vert s) \nu(s) \cdot \delta})}\}$ and
    $\tmix^{\pi}(z) := \lceil \log_{\rho}(\frac{\pi(a \vert s) \nu(s)}{2C}) \rceil$. 
\end{lemma}
\begin{proof}
    For notational convenience, let $\lambda^{\pi}(s,a) := \pi(a \vert s) \nu(s)$ and $b(z) := \frac{\lambda^\pi(z)}{2\tmix^\pi(z) \cdot \log(\frac{4\tmix^\pi(z)}{\lambda^\pi(z) \cdot \delta})}$, so $\beta_\delta^\pi(z) = \min\{ 1- \sqrt \gamma, \frac{b(z)}{3}\}$. We omit dependence on $\delta$ and $\pi$ since it is clear from context. We similarly omit $\pi$ and $m$ from $\tau^\pi_m$.

    The first inequality is a direct result of the fact ${\tQ}_m$ is an $(m-\tau(z))$-truncated unbiased estimator of $Q^\pi$, with a truncation error of $(1-\gamma)^{-1}\gamma^{m-\tau(z)}$.
    For the second bound, it was shown in~\cite[Lemma 3.1]{li2025policy} that $\Pr\{\tau(z)=i \vert Z_0\} \leq (1-\frac{\lambda^\pi(z)}{2})^{\lceil i/\tmix^\pi(z) \rceil -1}$. 
    Therefore,
    \begin{talign}
        &\Pr\{\tau(z) \geq b(z)^{-1} \vert Z_0\} \label{eq:hitting_time_concentration}
        \\ &\leq
        \sum_{i=\lceil b(z)^{-1} \rceil }^\infty (1-\frac{\lambda^\pi(z)}{2})^{\lceil i/\tmix^\pi(z) \rceil -1} 
        \nonumber \\
        &\leq
        2\sum_{i=\lceil b(z)^{-1} \rceil}^\infty (1-\frac{\lambda^\pi(z)}{2})^{i/\tmix^\pi(z)} \nonumber \\
        &\leq
        2\sum_{i=\lceil b(z)^{-1} \rceil}^\infty (1-\frac{\lambda^\pi(z)}{2\tmix^\pi(z)})^{i} 
        \nonumber \\
        &=
        \frac{4\tmix^\pi(z)}{\lambda^\pi(z)}(1-\frac{\lambda^\pi(z)}{2\tmix^\pi(z)})^{\lceil b(z)^{-1} \rceil} 
        \leq
        \delta, \nonumber
    \end{talign}
    where the penultimate line uses Bernoulli's inequality and the last bound is by choice in $b(z)$.

    For the remainder of the proof, let us condition on the complement event above, $\tau(z) < b(z)^{-1}$, which holds with probability $1-\delta$.
    We consider two phases of policy evaluation, one where the exploration time $m$ is small and the other where it is sufficiently large.

    \noindent \textbf{Case 1}: $m \leq 2b(z)^{-1}$: Then we have the crude bound,
        $\gamma^{m-\tau(z)}
        \leq
        1
        =
        3(1-\frac{2}{3})
        \leq
        3(1-\frac{b(z)}{3})^{2b(z)^{-1}}
        \leq
        3(1-\frac{b(z)}{3})^m$,
    where the second inequality is by Bernoulli's inequality.

    \noindent \textbf{Case 2}: $m \geq 2b(z)^{-1}$: Recalling we condition on $\tau(z) < b(z)^{-1}$, then
        $m-\tau(z)
        >
        m-b(z)^{-1}
        =
        m/2 + (m/2-b(z)^{-1})
        \geq
        m/2$.
    Hence, $\gamma^{m - \tau(z)} \leq \gamma^{m/2}$. 
    Combining both cases together, \begin{talign*}
        \gamma^{m - \tau(z)} 
        &\leq 3\max\{\gamma^{m/2}, (1-\frac{b(z)}{3})^m\} 
        \\
        &= 3(1 - \min\{1-\sqrt{\gamma}, \frac{b(z)}{3}\})^m, 
    \end{talign*}
    and the proof is complete after observing the ``min'' term is $\beta^\pi_\delta(z)$.
\end{proof}

Although the bound $3(1-\beta^\pi_\delta(z))^m$ in~\eqref{eq:hitting_time_bound} decreases linearly with $m$, it only becomes non-trivially smaller than $1$ once the trajectory length $m$ is greater than $3/b(z)$. 
This is because for a small $m$, some state-action pair $z \in \cZ$ is unlikely to be visited, yielding an inaccurate $\tQ_{m,\unpi}$.
While these properties are also shared by~\cite[Lemma 3.1]{li2025policy}, their bound on the bias depends on unknown algorithm-dependent parameters, while ours is an easy-to-compute bound.

To ensure the bias is no larger than $\varsigma \in (0, (1-\gamma)^{-1}]$ for an arbitrary state-action $z \in Z^\pi(\unpi)$, Lemma~\ref{lem:q_bias} suggests the following \textit{dynamic exploration time},
\begin{talign} \label{eq:dyn_mixing_time}
    \tmpi(\unpi, \varsigma) := 
    \max_{z \in Z^\pi(\unpi)} \tau_\infty^\pi(z) + \lceil \log_\gamma((1-\gamma)\varsigma) \rceil,
\end{talign}
where $\varsigma,\unpi > 0$ are some to-be-chosen constants and $\tPi^{\pi}(\cdot)$ is from~\eqref{eq:tilde_Pi_t_defn}.
Here, the name ``dynamic'' indicates this variable depends on the hitting time $\tau^\pi_\infty(z)$, which is a random variable depending on real-time observations.
Therefore, $\tmpi(\unpi, \varsigma)$ is also random.
Although Lemma~\ref{lem:q_bias} derives a high-probability bound on this random variable, it depends on the mixing time $\tmix^\pi(z)$, which is algorithm-dependent and can be arbitrarily large.
The next section derives novel algorithm-independent bounds.

\subsection{Implicit action exploration implies state exploration}  \label{sec:implicit_action_implies_exploration}
We introduce the only assumption made in this paper about the underlying MDP, which pertains to the optimal policy $\pi^*$.
Note that a deterministic optimal policy exists for MDPs without regularization~\cite{puterman2014markov}.
\begin{assumption} \label{asmp:optimal_mixing}
    There exists a deterministic optimal policy $\pi^*$ s.t.~the Markov chain defined by $P_{\pi^*}$ is irreducible and has a unique stationary state distribution of $\nu^*$ where $\nu^*(s) \geq \unu^* > 0$.
    Moreover, there are constants $C^* \geq 1$ and $\rho^* \in (0,1)$ satisfying~\eqref{eq:l1_mixing} w.r.t.~policy $\pi^*$.
\end{assumption}
This assumption only posits the existence of a deterministic, optimal policy $\pi^*$ that satisfies~\eqref{eq:l1_mixing}, while prior works may assume every generated policy $\pi_t$ or any policy in general satisfies~\eqref{eq:l1_mixing}~\cite{lan2023policy,li2025policy}. Thus, this assumption is algorithm-independent.
We will later weaken it to any optimal (including randomized) policy when incorporating function approximation.

Using Assumption~\ref{asmp:optimal_mixing}, we will show action exploration (i.e.,~$\unpi > 0$) implies state exploration by satisfying~\eqref{eq:l1_mixing}.
\begin{proposition} \label{prop:pi_t_mixing}
  Suppose Assumption~\ref{asmp:optimal_mixing} holds, and consider a policy $\pi$ satisfying 
  \begin{talign*}
    \unpi := \min_{s,a : \pi^*(a \vert s) =1} \{\pi(a \vert s) \} > 0.
  \end{talign*}
  Then the Markov chain defined by $P_\pi$ has a unique stationary state distribution $\nu$ and mixes in the sense of~\eqref{eq:l1_mixing} with parameters $C \in (4,8)$ and $\rho \in (0,1-\frac{\unpi^{\ob}(\unu^*)^2}{2\ob} )$, where
  \begin{talign*}
      \ob := \lceil \frac{\log(C^*/\unu^*)}{\log(1/\rho^*)} \rceil < +\infty,
      \quad
      \nu(s) 
      \geq 
      \frac{\unpi^{\ob} \unu^*}{2} > 0, \ \forall s \in \cS.
  \end{talign*}
\end{proposition}
\begin{proof}
    Let $a^*(s) \in \cA$ satisfy $\pi^*(a^*(s) \vert s) = 1, \ \forall s \in \cS$.
  Since $\frac{\pi(a^*(s) \vert s)}{\pi^*(a^*(s) \vert s)} \geq \unpi$, then for any $t \geq 0$, 
  \begin{talign}
    &P^{(t)}_\pi(q,s) \nonumber
    \\
    &= 
    \sum_{s' \in \cS, a' \in \cA} \cP(s \vert s', a') \pi(a' \vert s') P_\pi^{(t-1)}(q,s') \nonumber
    \\
    &\geq
    \unpi \sum_{s' \in \cS, a' \in \cA} \cP(s \vert s', a') \pi^*(a' \vert s') P^{(t-1)}_\pi(q,s') \nonumber
    \\
    &\geq
    (\unpi)^{t} \sum_{s' \in \cS, a' \in \cA} \cP(s \vert s', a') \pi^*(a' \vert s') P^{(t-1)}_{\pi^*}(q,s')  \nonumber
    \\
    &=
    (\unpi)^{t} P^{(t)}_{\pi^*}(q,s). \label{eq:ratio_argument}
  \end{talign}
  Since $\unpi > 0$, then one can combine~\eqref{eq:ratio_argument} with the same argument that showed the Markov chain associated with $P_{\pi^*}$ is irreducible (see~\cite[Theorem 4.9]{levin2017markov}) to also establish the irreducibility of ${P}_{\pi}$.
  Therefore, $P_\pi$ has a stationary distribution, denoted by $\nu$~\cite[Corollary 1.17]{levin2017markov}.
  Denote $\nu^* \in \mathbb{R}^{\vert \cS \vert}$ as the stationary distributions of $\pi^*$.
  We have
  \begin{talign}
    &(\underline{\nu}^*)^{-1} \cdot \nu^* 
    \geq
    \mathbf{1}_{\vert \cS \vert}
    \geq
    \nu^* P_\pi 
    =
    (\nu^* - \underline{\nu}^* \cdot \nu) P_\pi + \underline{\nu}^* \cdot \nu 
    \nonumber
    \\ &\geq
    (\nu^* - \underline{\nu}^* \cdot \mathbf{1}) P_\pi + \underline{\nu}^* \cdot \nu 
    \geq
    \underline{\nu}^* \cdot \nu, \label{eq:nu_star_nu_switch}
  \end{talign}
  where the second inequality holds by the Cauchy-Schwarz inequality, and the equality is due to the stationarity of $\nu$, or~$\nu P_\pi = \nu$. 
  Next, applying Assumption~\ref{asmp:optimal_mixing} tells us that for any $x,y \in \cS$ and $t \geq \ob$,
  \begin{talign}
    \vert P_{\pi^*}^{(t)}(x,y) - \nu^*(y) \vert
    &\leq
    \frac{1}{2}\|P_{\pi^*}^{(t)}(x, \cdot) - \nu^*\|_1
    \nonumber
    \\ &\leq
    \frac{C^* \cdot (\rho^*)^{t+1}}{2}
    \leq
    \frac{\underline{\nu}^*}{2}
    \leq
    \frac{\nu^*(y)}{2}, \label{eq:pi_star_mixing}
  \end{talign}
  where the first inequality is by the relation between TV-distance and $\|\cdot\|_1$, while the third is by the choice of $\ob$.
  The above implies $P_{\pi^*}^{(t)}(x,y)  > \underline{\nu}^*(y)/2$.
  Therefore, we have for any $t \geq \ob$
  \begin{talign} \label{eq:P_pi_t_relation_to_nu}
    P_\pi^{(t)}(x,y) 
    &\geq
    (\unpi)^{t} \cdot P_{\pi^*}^{(t)}(x,y)
    \\ &\geq
    (\unpi)^{t} \cdot \frac{\nu^*(y)}{2}
    \geq
    \frac{\unpi^{t} \cdot (\underline{\nu}^*)^2}{2} \cdot \nu(y), \nonumber
  \end{talign}
  where the first inequality can be shown similarly to~\eqref{eq:ratio_argument}, and the last is by~\eqref{eq:nu_star_nu_switch}.
  Let us denote $\delta' := \frac{\unpi^{\ob} \cdot (\underline{\nu}^*)^2}{2} \in (0,1/2)$ and $\theta' := 1-\delta'$.
  We also define a mixing rate $\rho$ as
  \begin{talign*}
      \rho 
      := 
      (\theta')^{1/\ob} 
      \leq 
      1-\frac{(\unu^*)^2\unpi^{\ob}}{2\ob} \in (\frac{1}{2},1),
  \end{talign*}
  where the inequality is by Bernoulli's inequality.
  Then by combining~\eqref{eq:P_pi_t_relation_to_nu} and the bound on $\rho$ with the proof from~\cite[Theorem 4.9]{levin2017markov}, we guarantee
    $\|P_\pi^{(t)}(x,\cdot) - \nu^{\pi}\|_{\mathrm{tv}} \leq C' \cdot (\rho')^\tau$,
  where $C := 1/\theta' \in (1,2)$.
  This bound on the TV-distance then implies~\eqref{eq:l1_mixing} after observing $2C' \cdot \rho^\tau \leq C \cdot \rho^{\tau+1}$, where $C := 4C'$.
  Next, to lower bound $\nu(s)$, we have for any $s \in \cS$,
  \begin{talign} \label{eq:nu_s_lb}
      \nu(s)
      &\stackrel{(1)}{=}
      \sum_{s' \in \cS} \nu(s') P^{(\ob)}_\pi(s,s')
      \\
      &\stackrel{(2)}{\geq}
      {\unpi}^{\ob} \sum_{s' \in \cS} \nu(s') P_{\pi^*}^{(\ob)}(s',s)
      \nonumber \\
      &\geq
      {\unpi}^{\ob} \sum_{s' \in \cS} [\nu^*(s) - \| P_{\pi^*}^{(\ob)}(s',\cdot) - \nu^*\|_{\mathrm{tv}}]\nu(s') 
      \nonumber \\
      &\stackrel{(3)}{\geq}
      {\unpi}^{\ob} \cdot \frac{\nu^*(s)}{2} \sum_{s' \in \cS} \nu(s')
      \geq
      {\unpi}^{\ob} \cdot \frac{\unu^*}{2}
      \nonumber
  \end{talign}
  where (1) is by stationarity of $\nu$, (2) is by~\eqref{eq:ratio_argument}, and (3) is by~\eqref{eq:pi_star_mixing}.
\end{proof}

By integrating this state exploration result with implicit action exploration~\cite{li2025policy}, we can derive a new algorithm-independent sample complexity next.

\subsection{Policy optimization with implicit state-action exploration} \label{sec:yanli_work}

We will now show our main result on the implicit exploration of SPMD (update~\eqref{eq:spmd_update}).
First, define the Bregman divergence w.r.t.~negative Tsallis entropy with entropic index $p \in (0,1)$~\cite{li2025policy},  
    $\omega(u): = -\frac{1}{(1-p)p} \sum_{a \in \cA} [u(a)]^p$. 
We have the following auxiliary result.
\begin{lemma} \label{lem:tsallis_strong}
    The modulus of strong convexity w.r.t.~$\|\cdot\|_1$ for $\omega(u)$ above is $\mu_\omega = 1$. Equivalently, $D^{\pi'}_{\pi}(s) \geq \frac{1}{2}\|\pi'(\cdot \vert s) - \pi(\cdot \vert s)\|_1^2$.
\end{lemma}
The constant $\mu_\omega$ is derived through an optimization problem with a closed-form optimal solution, found by the KKT conditions.
The full details are in~\ref{sec:proofs_for_yanli_work}.

The resulting Bregman divergence at state $s$ can be simplified into
\begin{talign} \label{eq:bregman_tsallis}
    &D^{\pi'}_{\pi}(s) 
    =
    \frac{1}{(1-p)p} \sum_{a \in \cA} \big(-[\pi'(a \vert s)]^p 
    \\& + (1-p)[\pi(a \vert s)]^p + p \pi'(a \vert s)[\pi(a \vert s)]^{p-1} \big).  \nonumber
\end{talign}
When $\pi_0$ is the uniform policy ($\pi_0(a \vert s) = 1/\vert \cA \vert$) and $\pi^*$ is a deterministic policy, then $D^{\pi^*}_{\pi_0}(s) = D^{\pi^*}_{\pi_0}(s) = \frac{-1 + \vert \cA \vert^{1-p}}{(1-p)p}$ for every $s \in \cS$.

We are now ready to establish the main convergence result.
Let us denote $\unu_t := \min_{s \in \cS} \nu_t(s)$ and $\rho_t$ as the mixing parameters (in the sense of~\eqref{eq:l1_mixing}) for an intermediate SPMD policy $\pi_t$.
\begin{theorem} \label{thm:implicit_convergence}
    Consider any accuracy threshold $\epsilon \in (0,\frac{1}{1-\gamma})$ and reliability level $\delta \in (0,1)$. 
    Suppose we run SPMD with stepsize $\eta_t = \alpha/\sqrt{k}$, where 
    \begin{talign} \label{eq:k_eps_defn}
        \alpha = \frac{\sqrt{\bar{D}_0}}{\sqrt{\hat{Q}^2 + M_h^2}}, 
        \quad 
        k =
        \big \lceil \frac{200\bar{D}_0 (\hat{Q}^2 + M_h^2)}{(1-\gamma)^2 \epsilon^2} \big \rceil,
    \end{talign}
    and $\bar{D}_0$ satisfies $\max_{s \in \cS} D^{\pi^*}_{\pi_0}(s) \leq \bar{D}_0$.
    In particular, suppose we choose $\pi_0$ as the uniform policy, a Bregman divergence w.r.t.~a negative Tsallis entropy with entropic index $p$, where
    \begin{talign*}
        p = \frac{1}{2+\log(1/[1-\gamma])},
        \quad
        \bar{D}_0 = \frac{-1+\vert \cA \vert^{1-p}}{(1-p)p},
        \quad 
        \hat{Q} = \frac{1}{1-\gamma},
    \end{talign*}
    and the TOMC estimator~\eqref{eq:montecarlo_Q} with dynamic exploration time $\tilde{m}^\pi(\varsigma, \unpi)$ from~\eqref{eq:dyn_mixing_time}, where
    \begin{talign*} 
        \varsigma &= \min\big\{\frac{(1-\gamma)\epsilon}{36}, \frac{2\hat{Q}\sqrt{2\log(2 \vert \cS \vert k/\delta)}}{\sqrt{k}} \big\}
        \\
        \unpi &= 
        \frac{\vert \cA \vert^{-1}(1-\gamma)p^2}{100\log(2\vert \cS \vert k/{\delta})}.
    \end{talign*}
    If Assumption~\ref{asmp:optimal_mixing} takes place, then with probability~$1-\delta$, both of the following occur.
    First, we have for all states $s \in \cS$,
    \begin{talign*}
        V^{\pi_{k}}(s) - V^{\pi^*}(s) \leq \epsilon \cdot \log(4k) \log(\frac{2 \vert \cS \vert\log(4k)}{\delta}).
    \end{talign*}
    Second, the total number of samples is bounded by 
    \begin{talign} \label{eq:paramfree_sampling_complexity}
        % \underbrace{\textstyle \sum_{t=0}^{k-1} \tilde{m}_t(\varsigma, \unpi)}_{\text{total samples}}
        % \leq
        &\overbrace{\tilde{O}\big\{ \tfrac{\vert \cA \vert^{2-p} (\hat{Q}^2 + M_h^2) [\log(1/\delta)]^2}{(1-\gamma)^3 (1-\orho) \unu\epsilon^2 } \big\}}^{\text{alg-dependent bound}}
        \\ &\leq
        \underbrace{\tilde{O}\big\{ \tfrac{\vert \cA \vert^{2-p} (\hat{Q}^2 + M_h^2)[\log(1/\delta)]^{2}}{(1-\gamma)^{3}\epsilon^2 } \cdot \Dexpl(\delta) \big\}}_{\text{alg-independent bound}}, \nonumber
    \end{talign}
      where $\unu := \min_{0 \leq t \leq k-1} \unu_t$ and $\orho = \max_{0 \leq t \leq k-1} \rho_t$, $\ob = \lceil \frac{\log(8/\nu^*)}{\log(1/\rho^*)} \rceil$ (Proposition~\ref{prop:pi_t_mixing}), 
    \begin{talign} \label{eq:Dexpl_defn}
        \Dexpl(\delta) := \big(\frac{\vert \cA \vert[\log(1/\delta)]}{1-\gamma}\big)^{2\ob} \cdot \frac{1}{(1-\rho^*)(\unu^*)^3},
    \end{talign}
    and $\tilde{O}(\cdot)$ hides algorithm-independent and polylogarithmic terms.
\end{theorem}
\begin{proof}
    We wish to apply Proposition~\ref{prop:stronger_gap_converge_with_relaxed}.
    To do so, we must verify its two conditions, one on the magnitude~\eqref{eq:second_moment} and the other the bias~\eqref{eq:weight_diff_bias}.
    \eqref{eq:second_moment} is satisfied with our choice $\hat{Q} = (1-\gamma)^{-1}$ by construction of the Monte-Carlo estimator~\eqref{eq:montecarlo_Q}.
    It remains to bound the bias.
    We will show for iterations $t=-1,0,1,\ldots,k-1$ (where $t=-1$ vacuously holds for ``iteration $t=-1$''),
    \begin{enumerate}
        \item \textbf{Item 1}: Inequality~\eqref{eq:weight_diff_bias} w.r.t.~$\pi_t$ and an upper bound of $2\varsigma$;
        \item \textbf{Item 2}: If $\pi_{t+1}(a \vert s) < \unpi$, then $\pi^*(a \vert s) = 0$.
    \end{enumerate}
    Items 1 and 2 can be shown similarly to Lemma 5.5 and Proposition 5.3 from~\citep{li2025policy}, respectively.
    We show the main steps here for completeness, which is done through mathematical induction.

    As briefly mentioned earlier, the ``base case'' $t=-1$ holds.
    Consider any $t \geq 0$, and suppose by induction Items 1 and 2 hold for every iteration up to $t-1$.
    Define $\delta_t(s,a) :=  \tQt(s,a) - \Qt(s,a)$. 
    At iteration $t$ and any $s \in \cS$, we have
    \begin{talign*}
        &\mathbb{E} \langle \delta_t(s,\cdot), \pi_t(\cdot \vert s) - \pi^*(\cdot \vert s) \rangle
        \\ &\stackrel{(1)}{=}
        \mathbb{E}_{a : \pi_t(a \vert s) \geq \unpi} \delta_t(s,\cdot) \cdot [\pi_t(a \vert s) - \pi^*(\cdot \vert s)]
        \\ &+
        \mathbb{E}_{a : \pi_t(a \vert s) < \unpi} \delta_t(s,\cdot) \cdot \pi_t(a \vert s) 
        \\
        &\stackrel{(2)}{\geq}
        -2\max_{a : \pi_t(a \vert s) \geq \unpi} \vert \mathbb{E}_{\xi_{t-1}} \delta_t(s,a) \vert 
        \stackrel{(3)}{\geq}
        2\varsigma.
    \end{talign*}
    Here, (1) uses the inductive hypothesis for Item 2; the first summand before (2) follows by the tower rule and the Cauchy-Schwarz inequality, while the second summand is by $\delta_t(s,a) = (1-\gamma)^{-1} - \Qt(s,a) \geq 0$, where we recall the definition of~$\tQt$ from~\eqref{eq:montecarlo_Q} for $z \not \in \tPi^{\pi_t}(\unpi)$; and (3) applies Lemma~\ref{lem:q_bias} and choice of the dynamic exploration time $\tm^{\pi_t}(\unpi,\varsigma)$ from~\eqref{eq:dyn_mixing_time}.
    Thus, we have shown Item 1.

    We now show Item 2 at iteration $t$. 
    Applying the second result from Proposition~\ref{prop:stronger_gap_converge_with_relaxed} for $\pi_t$ and recalling the choice of $\alpha$, $k$, and $\varsigma$, we have with probability $1-\delta/(2k)$,
    \begin{talign}
        (1-\gamma)D^{\pi^*}_{\pi_t}(s)
        &\leq
        \bar{D}_0 + \alpha^2 (\hat{Q}^2 + M_h^2)  \nonumber
        \\& 
        + 4\alpha \hat{Q}\sqrt{2\log(2\vert \cS \vert k/\delta)} + 2\alpha \varsigma \sqrt{k} 
        \nonumber
        \\
        &\leq
        2\bar{D}_0 + 8\sqrt{2\bar{D}_0 \log(2 \vert \cS \vert k/\delta)}
        \nonumber
        \\
        &\leq
        10(\bar{D}_0+1)\sqrt{2\log(2 \vert \cS \vert k/\delta)},
        \label{eq:dist_t_bound}
    \end{talign}
    where the second inequality used our choice of $\alpha$ and $\varsigma$.
    Since $\pi^*$ is assumed to be a deterministic optimal policy, let $a^*(s) \in \cA$ be the optimal action at state $s$, where $\pi^*(a^*(s) \vert s) = 1$.
    Combining~\eqref{eq:bregman_tsallis} with~\eqref{eq:dist_t_bound} and recalling $\bar{D}_0$, we deduce
    \begin{talign*}
        &\frac{(-1 + p \cdot [\pi_t(a^*(s) \vert s)]^{p-1})}{(1-p)p} 
        \\ &\leq 
        % D^{\pi^*}_{\pi_t}(s)
        % \\
        % \text{and} \quad 1 = \frac{(1-\gamma)^{\log(2)/\log((1-\gamma)^{-1})} \cdot 4^{1/2}}{[\log((1-\gamma)^{-1})]^0} &\leq \frac{(1-\gamma)^p \cdot 4^{1-p}}{[\log((1-\gamma)^{-1})]^{2p-1}}.
        % &\leq
        10\big(\frac{-1 + \vert \cA \vert^{1-p}}{(1-p)p} + 1 \big) \frac{\sqrt{2 \log(2 \vert \cS \vert k/\delta)}}{1-\gamma}
        \\ &\leq
        \big(\frac{-1 + 10\vert \cA \vert^{1-p}}{(1-p)p} \big) \frac{\sqrt{2 \log(2 \vert \cS \vert k/\delta)}}{1-\gamma},
    \end{talign*}
    where the last inequality is because $(1-p)p \leq 1/4$.
    The above then implies
    \begin{talign*}
        &[\pi_t(a^*(s) \vert s)]^{p-1}
        \\ &\leq
        \frac{10 \vert \cA \vert^{1-p}\sqrt{2 \log(2 \vert \cS \vert k/\delta)}}{(1-\gamma)p}
        % \\
        % &=
        % \vert \cA \vert^{1-p}\big((18)^2 \log(\frac{2\vert \cS \vert k}{\delta}) \big)^{1/2} \cdot \frac{\log((1-\gamma)^{-1})}{1-\gamma}
        % \leq
        % \big( \frac{10 \cdot \vert \cA \vert^{1-p} \sqrt{\log(2 \vert \cS \vert k/\delta)}}{(1-\gamma)p^2 } \big)^{1-p}
        % =
        \\ &\leq
        \big( \frac{10^2 \vert \cA \vert \cdot [2\log(2 \vert \cS \vert k/\delta)]}{(1-\gamma)p^2}\big)^{1-p}
        =
        \big(\frac{1}{\unpi}\big)^{1-p},
    \end{talign*}
    where the second inequality is by $p \in (0,1/2)$ and $(1-\gamma)p \geq [(1-\gamma)p^2]^{1-p}$, because
    \begin{talign*}
        \log\big(\frac{(1-\gamma)p}{[(1-\gamma)p^2]^{1-p}}\big)
        &=
        p \log(1-\gamma) - (1-2p)\log(p)
        \\ &=
        p \log(1-\gamma)[1 + \log(p)]
        \geq 0.
    \end{talign*}
    The above bound helps complete the proof by induction for Item 2.

    With the proof by induction completed, we have satisfied the hypothesis of Proposition~\ref{prop:stronger_gap_converge_with_relaxed}.
    By using $\bar{D}_0 \geq 1$ and noting our choices in $\alpha$, $k$, and $\varsigma$, then after applying Proposition~\ref{prop:stronger_gap_converge_with_relaxed}, we can show optimality gap of $V^{\pi_k}(s) - V^{\pi^*}(s) \leq \epsilon \log(4k)\log(\frac{2\vert \cS \vert \log(4k)}{\delta})$ with probability $1-\delta/2$.

    Finally, we establish the sample complexity.
    We start with the algorithmic-dependent bound on $\tmpi_t(\unpi,\varsigma)$.
    Define $\delta' := \delta/(2\vert \cZ \vert k)$ and recall $\lambda^{\pi_t}(z) = \nu_t(s) \cdot \pi_t(a \vert s)$. 
    \sloppy Also denote $M_1 := \frac{
        \max_{z \in \tPi^{\pi_t}(\varsigma,\unpi)} 
        2 \log(\frac{4C}{\lambda^{\pi_t}(z)}) \log \big( \frac{4\log(4C/\lambda^{\pi_t}(z))}{(1-\rho_t)\lambda^{\pi_t}(z) \delta'}\big)
    }{\lambda^{\pi_t}(z) \cdot (1-\rho_t)}$
    and $M_2 :=  \frac{
        2\log(\frac{4C}{\unu_t \unpi}) \log \big( \frac{4\log(4C/[\unu_t \unpi])}{(1-\rho_t)\unu_t \unpi \delta'}\big)
    }{\unu_t \unpi \cdot (1-\rho_t)}$.
    We have with probability $1-\delta'$,
    \begin{talign}
        \tmpi_t(\unpi,\varsigma)
        &\leq
        % \max_{z \in \tPi^{\pi_t}(\varsigma,\unpi)} \frac{\log \frac{36}{\varsigma \cdot(1-\gamma)}}{\log(1 - \beta^{\pi_t}_{\delta'}(z))}
        % \leq
        \max_{z \in \tPi^{\pi_t}(\varsigma, \unpi)} \{ {\log \frac{36}{\varsigma \cdot(1-\gamma)}}/{\beta^{\pi_t}_{\delta'}(z)}\}  \nonumber
        \\
        &\leq 
        \log (\frac{36}{\varsigma \cdot(1-\gamma)})
        \max\big\{ \frac{1}{1-\sqrt{\gamma}},
        M_1
        \big\} \nonumber
        \\
        &\leq
        \log (\frac{36}{\varsigma \cdot(1-\gamma)})
        \max\big\{ \frac{2}{1-\gamma},
        M_2
        \big\} 
        \nonumber
        \\ &=
        \tilde{O}\big\{ \frac{\vert \cA \vert (\log\frac{\vert \cZ \vert}{\epsilon \delta})^2}{(1-\gamma) \unu_t \cdot (1-\rho_t)}\big\}, \label{eq:mixing_time_concentration}
    \end{talign}
    where the first inequality follows from Lemma~\ref{lem:q_bias}, definition of $\tmpi_t(\unpi, \varsigma)$ in~\eqref{eq:dyn_mixing_time}, and $-\log(1-x) \geq x$ for $x > 0$; the second inequality is by definition of $\beta^{\pi_t}_{\delta'}$ in Lemma~\ref{lem:q_bias}; and the third inequality is by definition of $\tPi^{\pi_t}(\varsigma,\unpi)$ in~\eqref{eq:tilde_Pi_t_defn}.
    The algorithm-dependent bounds follows by multiplying the per-iteration sample cost above with the iteration count $k = \tilde{O}(\frac{\vert \cA \vert^{1-p}(\hat{Q}^2 + M_h^2)}{(1-\gamma)^2\epsilon^2})$.

    We now establish the algorithm-independent bound.
    We observe that at any iteration $t$,
    \begin{talign} \label{eq:alg_dep_to_alg_indep}
        \unu_t \cdot (1-\rho_t)
        &\geq
        \frac{\unpi^{\ob}(\unu^*)^2}{2\ob} \cdot \frac{\unpi^{\ob}\unu^*}{2}
        \\ &\geq
        \frac{\unpi^{2\ob}(\unu^*)^3 \log(1/\rho^*)}{8\log(4/\nu^*)}
        \geq
        \frac{\unpi^{2\ob}(\unu^*)^3 (1-\rho^*)}{8\log(4/\nu^*)}, \nonumber
    \end{talign}
    where the first inequality is by Proposition~\ref{prop:pi_t_mixing} and the second is by definition of $\ob$ and $\rho^* \geq 1/2$ (Assumption~\ref{asmp:optimal_mixing}). Finally, the probability of failure is by union bound over all the possible failures.
%
    % Finally, the probability of failure is at most $\delta$ by union bound, where $\delta/2$ comes is from optimality gap and another $\delta/2$ (by union bound over all $k$ iterations) from the total samples.
\end{proof}

This appears to be the first algorithm-independent $\tilde{O}(\epsilon^{-2})$ sample complexity for finite state and action discounted RL problems in the online, on-policy model.
In addition, the algorithm-dependent sample complexity is improved by $\vert \cA \vert^{1/2}/(1-\gamma)$ compared to prior bounds for implicit exploration~\cite[Proposition 5.5]{li2025policy}. 
This improvement is due to a new modulus of strong convexity for the Bregman divergence (Lemma~\ref{lem:tsallis_strong}) and a selection of $p$ that adapts to the discount factor $\gamma$.
Furthermore, the algorithm is parameter-free since it avoids a priori knowledge of $\unu_t$, $\rho_t$, $\unu^*$, $\rho^*$.

One deficiency of the dynamic exploration time is that it requires visiting every state.
We next consider function approximation to exploit low-dimensional features, and we develop its extension for the auto-exploration framework.

\section{Auto-exploration for RL with function approximation}\label{sec:spmd_linear}
Following the existing policy evaluation literature~\cite{li2018hyperband,li2024stochastic,mou2023optimal,lan2023policy}, we will show how to estimate the action-value function $Q^\pi$ using linear function approximation. 

\subsection{A new temporal difference operator} \label{sec:F_new}
Let $\phi(z) \in \mathbb{R}^{d}$ be a basis/feature vector for each state-action pair $z \in \cZ$ with feature size $d \leq \vert \cZ \vert$. 
See~\cite{sutton1998reinforcement} for general examples and the discussion after Theorem~\ref{thm:pmd_ctd} for a specific one. 
Define the feature matrix, $\Phi := [\phi(z_1), \ldots, \phi(z_{\vert \cZ \vert})]^T \in \mathbb{R}^{\vert \cZ \vert \times d}$, which we assume is full column rank without loss of generality. 
Consider the linear function approximator with weights $\theta \in \mathbb{R}^d$,
\begin{talign*}
    Q_{\theta} := \Phi \theta.
\end{talign*}
% We drop the dependence of $\pi$ from $Q_\theta$ since $\pi$ is fixed and clear from context.
For notational convenience, we may remove $\theta$ and pass the mathematical accent for $\theta$ onto $Q$. 
% Specific to this section, we write $Q_\theta$ as $Q$, $Q_{\theta'}$ as $Q'$, $Q_{\theta_t}$ as $Q_t$, and $Q_{\thetastr}$ as $\Qstr$. 

Our goal is to find weights $\thetastr$ such that $\Qstr = Q_{\thetastr}$ approximates $Q^\pi$ as well as possible. 
An optimality condition for such weights is satisfying the projected Bellman equation~\cite{sutton1998reinforcement,mou2023optimal,li2023accelerated}, 
\begin{talign} \label{eq:proj_bellman_eqn_og}
   \Qstr = \Phi[\Phi^T W \Phi]^{-1} \Phi^TW(\gamma P^\pi \Qstr + c).
\end{talign}
The transition kernel $P^\pi \in \mathbb{R}^{\vert \cZ \times \vert \cZ \vert}$ defined w.r.t.~policy $\pi$ is $P^\pi([s,a],[s',a']):= P(s' \vert s,a) \pi(a' \vert s')$.
Since $\Phi$ is full rank, the inverse exists if the diagonal weighting matrix $W = \mathrm{diag}(w)$ is positive for a distribution $w$ over the state-action space.
A standard choice of distribution is $w(z) =\nu^\pi(s) \cdot \pi(a \vert s)$, where $\nu^\pi$ is the stationary state distribution of $\pi$ and $z = (s,a)$. However, it may be difficult to sample from $\nu^\pi$, since the cost of sampling depends on the unknown mixing time~\cite{levin2017markov}. Instead, we select $w$ more flexibly with
\begin{talign} \label{eq:w_dist_defn} 
    w(z) &:= \kappa^{\piexplst}_{\sor}(s) \cdot \piexplact(a \vert s), 
\end{talign}
where the discounted state visitation distribution $\kappa^\pi_q(\cdot)$ is from~\eqref{eq:visitation_measure}, the \textit{origin state} $\sor \in \cS$ is an arbitrarily chosen state, and the \textit{exploration policy} $\piexplact$ is defined as
\begin{talign}
    \piexplact(a \vert s)
    &:= 
    (1-\epsact) \pi(a \vert s) + \frac{\epsact}{\vert \cA \vert}
    \label{eq:tpi_defn}
\end{talign}
for some user-defined exploration parameter $\epsact \in [0,1]$. 
We will later sample actions from $\tpi$ to induce both state and action exploration.
% The reason for two separate exploration parameters (and also two separate policies) is because $\epsact$ needs to be sufficiently small to ensure action exploration in a nearly on-policy manner, whereas $\epsst$ can be more aggressively set to yield an efficient and algorithm-independent sample complexity
% The separate policy $\piexplst$ is for off-policy sampling when the function approximation error is unknown, because the error introduces a bias that destroys the exploring properties of on-policy sampling using $\pi$.

Because our selection of $W = \mathrm{diag}(w)$ differs from the standard stationary distribution~\cite{mou2023optimal,li2023accelerated}, the structural properties of~\eqref{eq:proj_bellman_eqn_og} differ as well. 
To reconcile this difference, we will establish the well-posedness of~\eqref{eq:proj_bellman_eqn_og}.
Recall $\|x\|_W = \sqrt{x^TWx}$ for a vector $x \in \mathbb{R}^{\vert \cS \vert \cdot \vert \cA \vert}$.
\begin{proposition}\label{prop:proj_bellman_solvable}
    If $\min_{s \in \cS} \kappa^{\piexplst}_{\sor}(s) > 0$ and $\epsact \in (0,\frac{1-\gamma}{2}]$, then~\eqref{eq:proj_bellman_eqn_og} has a unique solution $\Qstr$. Moreover, $\|\Qstr\|_W \leq 4(1-\gamma)^{-1}$. 
\end{proposition}

This key result, whose proof is moved to~\ref{sec:missing_F_new_ec}, says that solving~\eqref{eq:proj_bellman_eqn_og} is equivalent to policy evaluation for estimating $Q^\pi$.
There are also two main advantages of utilizing $\kappa^\tpi_{\sor}$ over the stationary state distribution $\nu^\pi$.
One, $\kappa^{\piexplst}_{\sor}$ is always well-defined for finite states while $\nu^\pi$ may not exist for periodic Markov chains~\cite{levin2017markov}.
Thus, using $\kappa^{\piexplst}_{\sor}$ covers a more general class of Markov chains. 
Two, one can conveniently draw nearly unbiased samples from $\kappa^{\piexplst}_{\sor}$ as explained in the next subsection, while sampling from $\nu^\pi$ requires the unknown mixing rate.

While our goal is to solve~\eqref{eq:proj_bellman_eqn_og}, it involves an undesirable matrix inverse. To bypass this, we left multiply~\eqref{eq:proj_bellman_eqn_og} by $\Phi^TW$ to produce the following sufficient condition for the best weights $\thetastr$:
\begin{talign} \label{eq:lin_op}
    &F(\thetastr) = \mathbf{0},
    \quad \text{where} 
    \\ & F(\theta) := \Phi^TW(\Phi \theta - c - \gamma P^\pi \Phi \theta). \nonumber
\end{talign}
Solving $F(\theta) = \mathbf{0}$ is equivalent to solving a linear system. In fact, it turns out that $F(\thetastr) = \mathbf{0}$ is both necessary and sufficient to solve~\eqref{eq:proj_bellman_eqn_og}, as shown below. 
For notational convenience, we define
\begin{talign}
    \LamMax &:= \max_{z \in \cZ} \|\phi(z)\|_2^2 % \lambda_{\max}(\Phi^TW\Phi)
    % \quad \text{and} \quad 
    \nonumber
    \\
    \LamMin &:= \lambda_{\min}(\Phi^TW\Phi).
    \label{eq:uW_defn}
\end{talign}
Since $W = \mathrm{diag}(w)$ with distribution $w$, then $\LamMin \leq \lambda_{\max}(\Phi^T W \Phi) \leq \LamMax$.
We will now show the strong monotonicity of $F(\cdot)$.
The next two proofs are in~\ref{sec:missing_F_new_ec}.
\begin{lemma} \label{lem:monotone_with_disc_visit}
    Under the same hypothesis as Proposition~\ref{prop:proj_bellman_solvable}, then $\LamMin > 0$ and 
    for all $\theta,\theta' \in \mathbb{R}^d$,
        $\langle F(\theta) - F(\theta'), \theta - \theta' \rangle
        \geq
        \frac{1-\gamma}{4}\|Q - Q'\|_W^2
        \geq
        \frac{(1-\gamma)\LamMin}{4}\|\theta - \theta'\|_2^2$. 
\end{lemma}
The strong monotonicity property ensures $\thetastr$ is the unique solution to $F(\thetastr) = \mathbf{0}$.
We next establish the Lipschitz continuity for $F(\cdot)$.
\begin{lemma} \label{lem:F_lipschitz}
    Under the same hypothesis as Proposition~\ref{prop:proj_bellman_solvable}, then for all $\theta,\theta' \in \mathbb{R}^d$,
        $\|F(\theta) - F(\theta')\|_2 
        \leq 
        2\sqrt{\LamMax}\|Q-Q'\|_W 
        \leq 
        2\LamMax\|\theta-\theta'\|_2$. 
\end{lemma}
\noindent In other words, $F(\cdot)$ is an $L$-Lipschitz continuous and $\mu$-strongly monotone operator, where
\begin{talign} \label{eq:F_mu_L_defn}
    L := 2\LamMax \quad \text{and} \quad \mu := \frac{(1-\gamma)\LamMin}{4}.
\end{talign}
Combining Lemma~\ref{lem:monotone_with_disc_visit} with Proposition~\ref{prop:proj_bellman_solvable} yields the useful bound
\begin{talign} \label{eq:otheta_bnd}
    \|\otheta\|_2^2 \leq \oTheta^2 \quad \text{where} \quad \oTheta:= \frac{2}{\sqrt{(1-\gamma)\mu}}.
\end{talign}

\subsection{Constructing the stochastic temporal difference operator} \label{sec:sto_hF_defn}
We will now show how to construct a stochastic operator for~\eqref{eq:lin_op} in the online model.
At a high level, we will draw approximate samples of $w(\cdot)$ from~\eqref{eq:w_dist_defn}.
First, fix an arbitrary max trajectory length, $m \geq 1$.
\newline

\noindent \textbf{Step 1: Random geometric time}.
Draw an iid random variable $\tilde{t}$ according to $\tilde{t} \sim \mathrm{Geo}(1-\gamma)$.
If $\tilde t \geq m$, we set the stochastic operator as $\hF(\theta) = \mathbf{0}$ (see~\eqref{eq:hF_def}).
Otherwise, we proceed to Step 2.

\noindent \textbf{Step 2: Approximate sample of the discounted visitation distribution}.
Denote $s_0$ as the most recently observed state from the online model.
Starting at state $s_0$, we sample actions according to
\begin{talign*}
    a_t &\sim \piexplst(\cdot \vert s_t), \quad t = 0,\ldots,t'-1,
    \quad
\end{talign*}
where we write $t' := \tau^{\piexplst}(\sor) + \tilde t$ for notational convenience, and we recall $\tau^{\piexplst}(\sor) := \tau^{\tpi}_{\infty}(\sor)$ (we drop the subscript $\infty$ for simplicity) is the first hitting time of $\sor$ w.r.t.~the exploration policy $\piexplst$, defined in~\eqref{eq:hitting_time}.
The last observed state $s_{t'}$ is an approximate sample from $\kappa_{\sor}^{\piexplst}$.

\noindent \textbf{Step 3: Operator construction}.
From state $s_{t'}$, we build our stochastic operator $\hF(\cdot)$ identically to~\cite[Section 4.2.2]{li2024stochastic}, which we repeat for completeness.
Sample two more actions according to
\begin{talign*}
    a_{t'} \sim \piexplact(\cdot \vert s_{t'}) 
    \quad \text{and} \quad
    a_{t'+1} \sim \pi(\cdot \vert s_{t'+1}).
\end{talign*}
After applying action $a_{t'}$, we observe both the cost $c(s_{t'}, a_{t'})$ and the next state $s_{t'+1}$.
We collect all the observed quantities as $\xi_{t'} = \{(s_{t'},a_{t'}), (s_{t'+1},a_{t'+1}), c(s_{t'},a_{t'})\}$.
Then our stochastic operator, which depends on parameters ($\sor$, $m$, $\epsst$), is defined as
\begin{talign}\label{eq:hF_def}
    \hF(\theta) := 
        \tF(\theta, \xi_{t'}) \cdot \mathbf{1}[\tilde t < m],
\end{talign}
where we denote $z_t := (s_t,a_t)$ and $\tF(\theta, \xi_t) := \phi(z_{t})[\phi(z_{t})^T\theta - c(z_{t})- \gamma \phi(z_{t+1})^T \theta]$. Multiplication by the indicator variable corresponds to $\hF(\theta) = \mathbf{0}$ when $\tilde t \geq m$, as described in Step 1. 
This completes the construction of the stochastic operator $\hF(\cdot)$.
\newline

Observe that the number of samples to build $\hF(\cdot)$ is equal to the dynamic exploration time,
\begin{talign} \label{eq:dyn_mixing_time_fa}
    \tilde{m}^{\piexplst}(\sor, \tilde t, m) := [\tau^{\piexplst}(\sor) + \tilde t + 2] \cdot \mathbf{1}[\tilde t \leq m].
\end{talign}
Indeed, this dynamic exploration time is similar to the tabular case~\eqref{eq:dyn_mixing_time}, where both are random variables because they depend on the hitting time $\tau^{\tpi}(\cdot) := \tau^{\tpi}_\infty(\cdot)$ from~\eqref{eq:hitting_time}. 
But there are still a couple of key differences.
First,~\eqref{eq:dyn_mixing_time_fa} only requires visiting the fixed origin state $\sor$, while~\eqref{eq:dyn_mixing_time} must visit every state.
Second,~\eqref{eq:dyn_mixing_time_fa} is defined w.r.t.~the perturbed policy $\piexplst$, while~\eqref{eq:dyn_mixing_time} is defined w.r.t.~the policy $\pi$, which leads to off-policy sampling when $\epsst > 0$ since $\piexplst \ne \pi$.

Let us now discuss the statistical properties of $\hF(\cdot)$. We start by showing it is a nearly unbiased estimator of $F(\cdot)$ in the next lemma. 
Recall $L$ from~\eqref{eq:F_mu_L_defn}, the optimal weights $\thetastr$, and its norm upper bound $\oTheta$ from~\eqref{eq:otheta_bnd}. 
A proof for the following result can be found in~\ref{sec:missing_sto_hF_defn}.
\begin{lemma} \label{lem:F_bias}
    Denoting $C_1 := \vert \cZ \vert^2 L$, then for all $\theta,\theta' \in \mathbb{R}^d$, we have
    \begin{talign*}
        &\|\mathbb{E}[\hF(\theta) - \hF(\theta') \vert \tau^{\piexplst}(\sor)] - (F(\theta) - F(\theta'))\|_2
        \\
        &\leq
        C_1 \gamma^m\|\theta-\theta'\|_2.
    \end{talign*}
    If the hypothesis from Proposition~\ref{prop:proj_bellman_solvable} holds, then with $\sigma_1 := C_1 \oTheta$, we have
        $\|\mathbb{E}[\hF(\thetastr) \vert \tau^{\piexplst}(\sor)] - F(\thetastr)\|_2
        \leq
        \sigma_1 \cdot \gamma^m$.
\end{lemma}

\noindent We next show the estimator is light-tailed with the following absolute bound.
\begin{lemma} \label{lem:F_lighttail}
    Denoting $C_2 := 4L^2$, then for all $\theta,\theta' \in \mathbb{R}^d$, we have
    \begin{talign*}
        &\|[\hF(\theta) - \hF(\theta')] - [F(\theta) - F(\theta')]\|_2^2 
        \\
        &\leq
        C_2 \|\theta - \theta'\|_2^2.
    \end{talign*}
    If the hypothesis from Proposition~\ref{prop:proj_bellman_solvable} holds, then with $\sigma_2 := \sqrt{C_2}\oTheta$, we have
        $\|\hF(\thetastr) - F(\thetastr)\|_2^2 
        \leq
        \sigma_2^2$.
\end{lemma}
\begin{proof}
    The first inequality can be shown using the same proof for Lemma~\ref{lem:F_lipschitz} and $(a+b)^2 \leq 2(a^2+b^2)$. 
    The second can be shown similarly along with the upper bound on $\|\otheta\|_2$ from~\eqref{eq:otheta_bnd}.
\end{proof}

Finally, recall that the number of samples to construct $\hF(\cdot)$ is based on a dynamic exploration time defined in~\eqref{eq:dyn_mixing_time_fa}.
We will now show that this random quantity exhibits a light-tail behavior.
\begin{lemma}\label{lem:hF_cost}
    Consider estimating $\bar k$ operators $\hF(\cdot)$ as described in Steps 1-3 above for exploration policies $\{\piexplst_i\}_{i=0}^{\bar k-1}$.
    For each $i$, let $\sor^{(i)}$ be the origin state, let $\tilde{t}^{(i)}$ be the random variable sampled from Step 1, and let $m^{(i)}$ be the max trajectory.
    Then the total number of state-action pairs observed is
        $B_{\bar k} := \sum_{i=0}^{\bar k-1} \tilde{m}^{\piexplst_i}(\sor^{(i)}, \tilde{t}^{(i)}, m^{(i)})$,
    which is the sum of dynamic exploration times from~\eqref{eq:dyn_mixing_time_fa} over all $\bar k$ policies.
    Furthermore,
    \begin{talign*}
        \Pr\big\{ B_{\bar k }
        > 
        \sum_{i=0}^{\bar k-1} (m^{(i)} + 1) + \tilde{O}(\frac{\bar k \cdot \log(\bar k/\delta)}{(1-\torho) \tunu_{or}} ) \big
        \}
        \leq
        \delta,
    \end{talign*}
    where $\torho := \max_{0 \leq i \leq \bar k-1} \rho^{\piexplst_i}$ and $\tunu_{or} := \min_{0 \leq i \leq \bar k-1} \nu^{\piexplst_i}(\sor^{(i)})$. 
\end{lemma}
\begin{proof}
    The dynamic exploration time $\tilde{m}^{\piexplst_i}(\sor^{(i)}, \tilde{t}^{(i)}, m^{(i)})$ is at most the sum of the random hitting time $\tau^{\piexplst}(\sor)$ and random truncation length $\tilde t + 2$.
    The hitting times can be bounded using~\eqref{eq:hitting_time_concentration}.
    Because we do not sample if $\tilde t \geq m^{(i)}$ (Step 1), then the latter summand has an absolute bound of $\tilde{m}^{(i)} + 1$.  
\end{proof}

The next section repeatedly builds this stochastic operator to solve the fixed-point equation~\eqref{eq:lin_op}.

\subsection{A conditional temporal difference with high-probability convergence} \label{sec:anytime_ctd}
Our policy evaluation algorithm is conditional temporal differencing (CTD), shown in Algorithm~\ref{alg:simple_ctd}. 
The algorithm iteratively updates the variable $\theta_\ell$, used to define the function approximation $Q_\ell = \Phi \theta_\ell$, by subtracting from it the stochastic operator $\hF(\theta_\ell)$ scaled by a stepsize $\iota_\ell$ over $N$ iterations.
Now, the construction of $\hF(\cdot)$ depends on four algorithmic parameters.
State $\sor$ is an arbitrarily chosen origin state.
The name ``conditional'' means that convergence is guaranteed if the algorithmic parameter $m_\ell$ satisfies a condition defined in Lemma~\ref{lem:ctd_abs_err}, similar to~\cite{kotsalis2022simple}.
The third input, $\epsst$, defines the exploration policy, $\tpi$.
While this CTD method is based on~\cite{kotsalis2022simple,lan2023policy}, the analysis is novel, since we derive high-probability bounds and use a new operator estimator.
\begin{algorithm}
\caption{}
\label{alg:simple_ctd}
\begin{algorithmic}[1]
\Procedure{CTD}{$N$, $\sor$, $\{m_\ell\}_\ell$, $\epsst$}
    % \State{Choose any state $\sor \in \cS$} \label{line:sor_select}
    \For{$\ell=0,\ldots,N-1$} \Comment{Initialize $\theta_0 = \mathbf{0} \in \mathbb{R}^d$}
        \State{Estimate $\hF(\theta_\ell)$ using parameters ($\sor$, $m_\ell$, $\epsst$)} \Comment{See~\eqref{eq:hF_def}}
        \State{Update $\theta_{\ell+1} = \theta_\ell - \iota_\ell \hF(\theta_\ell)$}
    \EndFor
    \State{\Return ${Q}_{N} := \Phi \theta_N$}
\EndProcedure
\end{algorithmic}
\end{algorithm}

We will now show that the distance to optimality decreases at a sublinear rate of $O(N^{-1/2})$.
The result will make use of constants $C_1,\sigma_1$ from Lemma~\ref{lem:F_bias}, $C_2,\sigma_2$ from Lemma~\ref{lem:F_lighttail}, as well as $L,\mu$ from~\eqref{eq:F_mu_L_defn} and $\oTheta$ from~\eqref{eq:otheta_bnd}.
Proofs for the following technical results can be derived from existing works~\cite{kotsalis2022simple,lan2023policy}, and we provide details in~\ref{sec:missing_anytime_ctd}.
\begin{lemma} \label{lem:ctd_abs_err}
    Suppose the hypothesis from Proposition~\ref{prop:proj_bellman_solvable} holds. 
    Choose algorithmic parameters
    \begin{talign*}
        &\ell_0 := \max\{\frac{8L^2}{\mu^2}, \frac{32C_2}{\mu^2}\}
        \quad \iota_\ell := \frac{2}{(\ell_0+\ell)\mu}, 
        \\
        &\text{and} 
        \quad \unm_\ell := \frac{\log(1/\mu) + \log(16C_1) + \log(\sigma_1/[\iota_\ell \sqrt{C_2}])}{\log(1/\gamma)}.
    \end{talign*}
    If $m_\ell \geq \unm_\ell$, then for any $\delta \in (0,1)$,
    \begin{talign*}
        \Pr\{\exists k \in [N] ~:~ \|\theta_k - \otheta\|_2^2
        >
        \mathcal{B}(k,\delta) \}
        \leq \delta,
    \end{talign*}
    where $\mathcal{B}(k,\delta) := \frac{2(\ell_0+1)(\ell_0+2)\|\theta_0 - \otheta\|_2^2}{(k+\ell_0+1)(k+\ell_0+2)} + \frac{16\sigma^2}{k\mu^2} + \frac{16\sqrt{C_2R_{[\delta]}^2\log(\frac{N}{\delta})}}{\sqrt{k}\mu}$, $\sigma := \sqrt{2\sigma_2^2+1}$. 
    Moreover, we have the absolute bound $\mathcal{B}(k,\delta) \leq (R_{[\delta]})^2 := 2\oTheta^2 + \frac{16\sigma^2}{\mu^2} + \frac{32\sqrt{C_2 \log(\frac{N}{\delta})}}{\mu}$.
\end{lemma}

Next, we will show the bias of $\theta_N$ can be made smaller at a faster linear rate.
\begin{lemma} \label{lem:ctd_bias}
    Consider the same setup as Lemma~\ref{lem:ctd_abs_err}, where we denote $R := R_{[1]}$. Then for any $\ell \leq N$,
    \begin{talign*}
        &\|\mathbb{E}\theta_{\ell} - \otheta\|_2^2
        \\ &\leq
        8^{-\lfloor \ell/ \ell_0 \rfloor}\|\theta_0 - \otheta\|_2^2 + \frac{16C_1R^2 \gamma^m}{\mu} + \frac{4C_1^2 R^2 \gamma^{2m}}{\mu^2}.
    \end{talign*}
\end{lemma}
Note that we still need to estimate the unknown $\mu$ to specify algorithmic parameters $\ell_0$, $\iota_\ell$, and $\unm_\ell$ in CTD.
The next section shows how to do this with little to no a priori knowledge of $\mu$.

\subsection{Policy optimization with linear function approximation} \label{sec:auto_linfun}
We now consider stochastic policy mirror descent (SPMD) from~\eqref{eq:spmd_update}, where $Q^{\pi_t}$ is estimated by CTD (Algorithm~\ref{alg:simple_ctd}) for policy evaluation.
As previously mentioned, choosing algorithmic parameters in CTD depends on the unknown strong monotonicity constant $\mu$ from~\eqref{eq:F_mu_L_defn}.
To estimate this parameter, suppose instead we have access to a lower bound $\ukappa$ s.t.
\begin{talign} \label{eq:ukappasor_defn}
    \ukappa \leq \min_{0 \leq t \leq k-1} \min_{s \in \cS} \kappa^{\piexplst_t}_{\sor}(s),
\end{talign}
where $k$ is the number of SPMD iterations, $\sor \in \cS$ is a user-chosen origin state, and $\tpi_t$ is the exploration policy of~policy $\pi_t$.
Although $\ukappa$ appears to be algorithm-dependent since it is defined w.r.t.~$\pi_t$, we will show it admits an algorithm-independent bound after Theorem~\ref{thm:pmd_ctd}.
This bound is omitted for now to simplify the result.

With some abuse of notation, we denote $\oTheta(\ukappa) := \frac{2}{\sqrt{(1-\gamma)\mu(\ukappa)}}$, which is exactly as $\oTheta$ is defined in~\eqref{eq:otheta_bnd} but with $\mu$ replaced by its lower bound estimate
\begin{talign*}
    \mu(\ukappa) := \frac{(1-\gamma)\unLam(\ukappa)}{4} \quad \text{and} \quad \unLam(\ukappa) := \ukappa \cdot \lambda_{\min}(\Phi^T \Phi).
\end{talign*}
Similarly, parameters $\ell_0(\ukappa)$, $\unm_\ell(\ukappa)$, $\sigma(\ukappa)$, and $R_{[\delta]}(\ukappa)$ are exactly as defined in Lemma~\ref{lem:ctd_abs_err} but with any dependence on the unknown $\mu$ replaced by $\mu(\ukappa)$, and likewise with $R_{[\delta]}(\ukappa)$ from Lemma~\ref{lem:ctd_bias}. Parameter $R_{[\delta]}(\cdot)$ depends on $\log(N)$, but we omit its dependence on $N$ for notational convenience.
Then, we set the CTD iterations $N$ as the smallest integer satisfying
\begin{talign}
    N &\geq \max\big\{B_1(\ukappa, N, \delta), ~~ 2\ell_0(\ukappa) \log_8\big( \frac{72 \sqrt{3} \oTheta(\ukappa)^2}{(1-\gamma)\epsilon} \big)\big\}, \label{eq:ctd_setup_params}
\end{talign}
where $\epsilon \in (0,(1-\gamma)^{-1})$ and $\delta \in (0,1)$ are user-defined, and 
    $B_1(\ukappa, N, \delta) := 
      \frac{\sqrt{6(1-\gamma)^2 (\ell_0(\ukappa)+2)}}{\sqrt{\oTheta(\ukappa)^2}}
      + \frac{(48)^2 (1-\gamma)^2 \sigma(\ukappa)^2}{\mu(\ukappa)^2}
      + \frac{(48)^2(1-\gamma)^4[R_{[\delta]}(\ukappa)]^2C_2 \log(\frac{4N}{\delta})}{\mu(\ukappa)^2}$,
and we recall $C_1 = \vert \cZ \vert^{2}L$ and $C_2 = 4L^2$.
At iteration $\ell$ of CTD, parameter $m_\ell$ is set to
\begin{talign} \label{eq:ctd_setup_params_b}
    m_\ell
    :=
    \Big \lceil \unm_\ell(\ukappa) + \frac{\log(4C_1 R(\ukappa)) + \log(1/\mu(\ukappa)) + \log(\frac{72\sqrt{3}}{(1-\gamma)\epsilon}) }{\log(1/\gamma)} \Big\rceil.
\end{talign}
With these parameters, we will establish the convergence of SPMD with function approximation.
We first introduce some notation for the $t$-iteration of SPMD.
% Let $\LamMin^{(t)}$ be defined as in~\eqref{eq:uW_defn}, $\tQ^{(t)}$ as the output of CTD, and $\oQ^{(t)}$ as the best estimator w.r.t.~policy $\pi_t$, where $\tQ^{(t)} = \Phi \theta^{(t)}$ and $\oQ^{(t)} = \Phi \otheta^{(t)}$.
Let $\oQ^{(t)} = \Phi \otheta^{(t)}$ be the best estimator w.r.t.~policy $\pi_t$.

\begin{theorem} \label{thm:pmd_ctd}
    Consider any $\epsilon \in (0,\frac{1}{1-\gamma})$, $\delta \in (0,1)$, and $\ukappa > 0$.
    Run SPMD with parameters $k$ and $\eta$ defined as in~\eqref{eq:k_eps_defn}, where we set $\hat{Q}^2 = \frac{6}{(1-\gamma)^{2}} + \frac{102\|\Phi\|_2^2}{(1-\gamma)^2\unLam(\ukappa)^2}$ and choose any Bregman divergence.
    Meanwhile, let us set CTD parameters $N$ to~\eqref{eq:ctd_setup_params}, $\sor \in \cS$ as an arbitrary state, $m$ to~\eqref{eq:ctd_setup_params_b}, and $\epsact = \frac{1-\gamma}{2}$.
    % , and we will specify $\epsst$ below.
    If Assumption~\ref{asmp:optimal_mixing} takes place and $\ukappa$ satisfies~\eqref{eq:ukappasor_defn}, then with probability $1-\delta$, both of the following occur.
    First, we have for all states $s \in \cS$,
    \begin{talign} 
        V^{\pi_k}(s) - V^{\pi^*}(s)
        \leq
        \big[\epsilon + \frac{36\epsapproxinf}{1-\gamma}\big](\log \frac{8k \vert \cS \vert}{\delta})^2,
        \label{eq:spmdctd_case_2}
    \end{talign}
    where $\epsapproxinf := \max_{0 \leq t < k}\|Q^{\pi_t} - \oQ_t\|_\infty$.
    Second, the total number of samples is at most 
    \begin{talign*}
        &\overbrace{\tilde{O}\big(
        \tfrac{\bar{D}_0 (\hat{Q}^2 + M_h^2) [\log(1/\delta)]^{3/2}}{(1-\gamma)^5 \ukappa^6 \epsilon^2} \big[ \tfrac{1}{1-\gamma} + \tfrac{\log(1/\delta)}{(1-\torho)\tunu_{or}} \big] \big)}^{\text{alg-dependent bound}}
        \\ &\leq
        \underbrace{\tilde{O}\big\{ \tfrac{\bar{D}_0 (\hat{Q}^2 + M_h^2) [\log(1/\delta)]^{5/2}}{(1-\gamma)^{5}\ukappa^6 \epsilon^2 } \cdot \Dexpl(\delta) \big\}}_{\text{alg-independent bound}},
    \end{talign*}
    \sloppy where $\max_{s \in \cS} D^{\pi^*}_{\pi_0}(s) \leq \bar{D}_0$;
    $\Dexpl(\cdot)$ is from~\eqref{eq:Dexpl_defn}; $\torho := \max_{0 \leq i \leq  k-1} \rho^{\piexplst_i}$ and $\tunu_{or} := \min_{0 \leq i \leq k-1} \nu^{\piexplst_i}(\sor)$; and $\tilde{O}(\cdot)$ hides dependence on algorithm-independent and polylogarithmic terms.
\end{theorem}
\begin{proof}
    Let $\LamMin^{(t)}$ be defined as in~\eqref{eq:uW_defn}, $\tQ^{(t)}$ as the output of CTD, and recall $\oQ^{(t)}$ as the best estimator w.r.t.~policy $\pi_t$, where $\tQ^{(t)} = \Phi \theta^{(t)}$ and $\oQ^{(t)} = \Phi \otheta^{(t)}$.
    In view of Lemma~\ref{lem:ctd_abs_err} and Lemma~\ref{lem:ctd_bias}, the choice of $N$ and $m$ provides the guarantees,
    \begin{talign*}
        \mathbb{E}\|\tQ^{(t)} - \oQ^{(t)}\|_\infty^2
        &\leq
        \|\Phi\|_2^2 \|\theta^{(t)} - \otheta^{(t)}\|_2^2
        \leq
        \frac{1}{(1-\gamma)^2}
        \\
        \|\mathbb{E}\tQ^{(t)} - \oQ^{(t)}\|_\infty
        &\leq
        \|\Phi\|_2 \|\mathbb{E}\theta^{(t)} - \otheta^{(t)}\|_2
        \leq
        \frac{(1-\gamma)\epsilon}{72},
    \end{talign*}
    which occurs with probability $1-\delta/4$ across all SPMD iterations $t=0,\ldots,k-1$.
    % Combining this with the facts $\max\{\epsilon, \|Q^\pi\|_\infty\} \leq (1-\gamma)^{-1}$ and the assumption~\eqref{eq:epsapproxinf_defn}, we get
    % \begin{talign}
    %     &\|\tQ^{(t)}\|_\infty^2 
    %     \leq
    %     3\|\tQ^{(t)} - \oQ^{(t)}\|_\infty^2 + 3\|\oQ^{(t)} - Q^{\pi_t}\|_\infty^2 
    %     \nonumber
    %     \\ &+ 3\|Q^{\pi_t}\|_\infty^2
    %     \leq
    %     \frac{7}{(1-\gamma)^{2}} \nonumber
    %     \\
    %     &\|\mathbb{E} \tQ^{(t)} - Q^{\pi_t}\|_\infty
    %     \leq
    %     \|\mathbb{E} \tQ^{(t)} - \oQ^{(t)}\|_\infty 
    %     \nonumber 
    %     \\ &
    %     + \|\oQ^{(t)} - Q^{\pi_t}\|_\infty \frac{(1-\gamma)\epsilon}{36}. \label{eq:ctd_fa_total_bias}
    % \end{talign}
    % Then we can use an identical argument as the tabular case (Theorem~\ref{thm:implicit_convergence}) to establish~\eqref{eq:spmdctd_case_1}.
% 
    % To prove the second convergence result~\eqref{eq:spmdctd_case_2}, we bound the approximation error differently as
    Now, the approximation error can be bounded by
    \begin{talign*}
        \|\oQ^{(t)} - Q^{\pi_t}\|_\infty^2
        &\leq
        \frac{\|\Phi\|_2^2}{\LamMin^{(t)}}\|\oQ^{(t)} - Q^{\pi_t}\|_W^2
        \\ &\leq
        \frac{2\|\Phi\|_2^2}{\unLam(\ukappa)}[\|\oQ^{(t)}\|_W^2 + \|Q^{\pi_t}\|_W^2]
        \\ &\leq
        \frac{34\|\Phi\|_2^2}{(1-\gamma)^2\unLam(\ukappa)},
    \end{talign*}
    where the first inequality is $\|\cdot\|_\infty \leq \|\cdot\|_2$ with Lemma~\ref{lem:F_lipschitz} and Lemma~\ref{lem:monotone_with_disc_visit}, 
    the second inequality by $\unLam(\ukappa) \leq \LamMin$, 
    and the last inequality is with the help of Proposition~\ref{prop:proj_bellman_solvable}.
    With this bound in hand, we then have % we then modify~\eqref{eq:ctd_fa_total_bias} into
    \begin{talign*}
        &\|\tQ^{(t)}\|_\infty^2 
        \leq
        3\mathbb{E}\|\tQ^{(t)} - \oQ^{(t)}\|_\infty^2 + 3\|\oQ^{(t)} - Q^{\pi_t}\|_\infty^2 
        \\ & + 3\|Q^{\pi_t}\|_\infty^2
        \leq
        \frac{6}{(1-\gamma)^{2}} + \frac{102\|\Phi\|_2^2}{(1-\gamma)^2\unLam(\ukappa)}
        \\
        &\|\mathbb{E} \tQ^{(t)} - Q^{\pi_t}\|_\infty
        \leq
        \|\mathbb{E} \tQ^{(t)} - \oQ^{(t)}\|_\infty 
        \\ &+ \|\oQ^{(t)} - Q^{\pi_t}\|_\infty
        \leq
        \frac{(1-\gamma)\epsilon}{72} + \epsapproxinf.
    \end{talign*}
    We obtain~\eqref{eq:spmdctd_case_2} after plugging these bounds into Proposition~\ref{prop:stronger_gap_converge_with_relaxed}.

    Finally, our choice of algorithmic parameters can be bounded as
        $k = O\{\frac{\bar{D}_0 (\oQ^2 + M_h^2)}{(1-\gamma)^2 \epsilon^2}  \}$,
        $N = \tilde{O}\{\frac{[\log(1/\delta)]^{3/2}}{(1-\gamma)^3 \ukappa^6} \}$,
        and
        $m_\ell = \tilde{O}\{ \frac{1}{1-\gamma}\}$,
    and the overall algorithm-dependent sample complexity can be deduced from Lemma~\ref{lem:hF_cost} with $\bar{k} := k \cdot N$.
    The algorithm-independent bound can be shown exactly as in the proof for Theorem~\ref{thm:implicit_convergence} (which requires Assumption~\ref{asmp:optimal_mixing}), where our choice of $\epspi = (1-\gamma)/2$ ensures that $\tpi(a \vert s) \geq (1-\gamma)/(2\vert \cA \vert)$ for all $(s,a) \in \cZ$.
    % as long as $\unpi$ from Proposition~\ref{prop:pi_t_mixing} satisfies $\unpi \geq \epsst/\vert \cA \vert$, where $\epsst$ is defined as~\eqref{eq:ctd_setup_params_2}.
    % The relation $\unpi \geq \epsst/\vert \cA \vert$
    % holds when~\eqref{eq:epsapproxinf_defn} occurs and when choosing a Bregman divergence with Tsallis entropy, which can be shown using the proof for Theorem~\ref{thm:implicit_convergence}.
    % For arbitrary $\epsapproxinf$, the relation is immediate by choice of parameter~\eqref{eq:ctd_setup_params_2}.
\end{proof}

% From the theorem, the choice of $\epsst$ depends on whether the function approximation error is known. 
% When it is known, we set $\epsst=0$ (implicit exploration) and otherwise $\epsst> 0$ (explicit exploration). 
A few remarks are in order.
First, Assumption~\ref{asmp:optimal_mixing} can be relaxed from an optimal deterministic policy to any optimal (including randomized) policy, which can be shown by integrating explicit exploration ($\epspi > 0$) into Proposition~\ref{prop:pi_t_mixing}.
It is also possible to achieve auto-exploration in a nearly on-policy manner ($\epspi=0$) when the function approximation error $\epsapprox$ is smaller than $(1-\gamma)\epsilon/36$. 
But the sampling approach is more complicated, so we omit it for simplicity.

Second, we can derive algorithm-independent bounds for $\ukappa$.
For instance, suppose $\sor$ is set to a uniform random state, $\sor \sim \mathrm{Uni}(\cS)$, resulting in $\kappa^{\tpi}_{\mathrm{Uni}(s)}(s) := \vert \cS \vert^{-1} \sum_{q \in \cS} \kappa_{q}^{\tpi}(s) \geq (1-\gamma)/\vert \cS \vert$.
Thus, we can set $\ukappa = (1-\gamma)/\vert \cS \vert$.
However, this choice of $\sor$ may require visiting every state.
Alternatively, consider any fixed $\sor \in \cS$.
A proof for the following result is in~\ref{sec:missing_auto_linfun}.
\begin{corollary} \label{cor:kappa_lb}
    Recall $\unu^* = \min_{s \in \cS} \nu^*(s)$ and $\ob = \lceil \frac{\log(C^*/\unu^*)}{\log(1/\rho^*)} \rceil$ from Proposition~\ref{prop:pi_t_mixing}.
    Under the same setup as Theorem~\ref{thm:pmd_ctd}, we have for any iteration $t$ and $\sor,s \in \cS$,
    \begin{talign*}
        \kappa^{\tpi_t}_{\sor}(s) 
        \geq 
        \frac{(1-\gamma)\unu^*}{2}\big( \frac{(1-\gamma)\gamma}{\vert \cA \vert}\big)^{\ob}.
    \end{talign*}
\end{corollary}
While $\ukappa$ can be set as the lower bound in the corollary above, this bound depends on the unknown $\unu^*$.
We can guess-and-check different $\ukappa$'s, starting from $\ukappa = 1$. 
After running SPMD+CTD with some $\ukappa$, we estimate the optimality gap using the so-called advantage gap function, $g^{\pi}(s) := \max_{p \in \DA} \{-\psi^\pi(s,p) \}$ (where $\psi$ is from Lemma~\ref{lem:performance_diff_deter}), which provides the tight bounds, 
\begin{talign*}
    g^\pi(s) \leq V^{\pi}(s) - V^{\pi^*}(s) \leq (1-\gamma)^{-1} g^\pi(s).
\end{talign*}
See~\cite{ju2024strongly} for the derivation and estimation error bounds, since $g^{\pi}$ can only be estimated in the online setting.
Now, once $g^\pi(s) \leq O\{(1-\gamma)\epsilon\}$, we can terminate. 
Otherwise, we halve $\ukappa$ and re-run SPMD.
As long as the function approximation error is not too large and satisfies $\epsapproxinf \leq \frac{(1-\gamma)\epsilon}{36}$, this guess-and-check will terminate in view of the upper bound on $g^\pi(s)$ once $\ukappa$ satisfies~\eqref{eq:ukappasor_defn}.
Note that it is possible to satisfy $\epsapproxinf = 0$ using linear basis functions. 
Let $\Phi \in \mathbb{R}^{\vert \cZ \vert \times \vert \cZ \vert}$ be a random matrix with iid~Gaussian entries.
Since $\Phi$ is nonsingular with probability 1~\cite{rudelson2008invertibility}, we ensure $\mathrm{span}(\Phi) = \mathbb{R}^{\vert \cZ \vert}$. Consequently, $\epsapproxinf = 0$.

Having established the convergence of SPMD, we next examine its numerical performance.

\section{Numerical experiments} \label{sec:experiments}
We consider two sets of environments. 
One, when there is efficient state exploration, and two, when there is not.
First, we describe the algorithms being tested.

\subsection{Experimental setup}
For the tabular setting (i.e.,~without function approximation), we implemented SPMD MC-Est and SPMD MC-Dyn. 
These two SPMD methods use the same Monte Carlo estimator~\eqref{eq:montecarlo_Q} but with different trajectory lengths. 
These are, respectively, a length computed by estimating the mixing properties (using~\cite{hsu2019mixing}) and our dynamic exploration time from~\eqref{eq:dyn_mixing_time}.
Likewise, for linear function approximation (using matrices with iid Gaussian entries), we implemented SPMD+CTD-Est and SPMD+CTD-Dyn. 
The former sets the state-action distribution $W$ in~\eqref{eq:lin_op} using the stationary distribution and estimates it (using~\cite{hsu2019mixing}), while the latter uses our discounted visitation distribution from~\eqref{eq:w_dist_defn}. 
We also implemented SARSA~\cite{sutton1998reinforcement}, a variant of on-policy Q-learning with a tunable $\epsilon$-exploration parameter.
Algorithmic parameters are tuned using grid search. 
More details can be found in our open-source code at \url{https://github.com/jucaleb4/pg-termination}.

\subsection{Performance on environments with efficient state exploration}

\begin{figure}
    \begin{subfigure}{.49\textwidth}
    \centering
    \includegraphics[width=.75\linewidth]{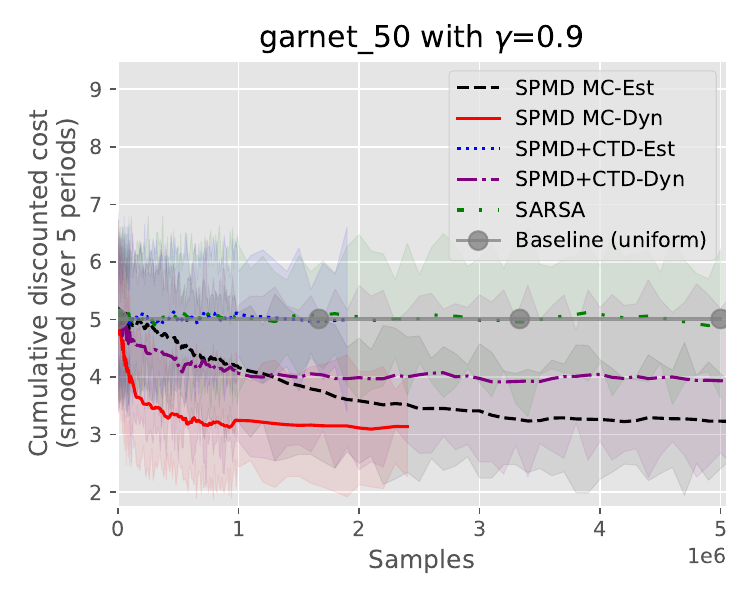}
    \end{subfigure}
    \begin{subfigure}{.49\textwidth}
    \centering
    \includegraphics[width=.75\linewidth]{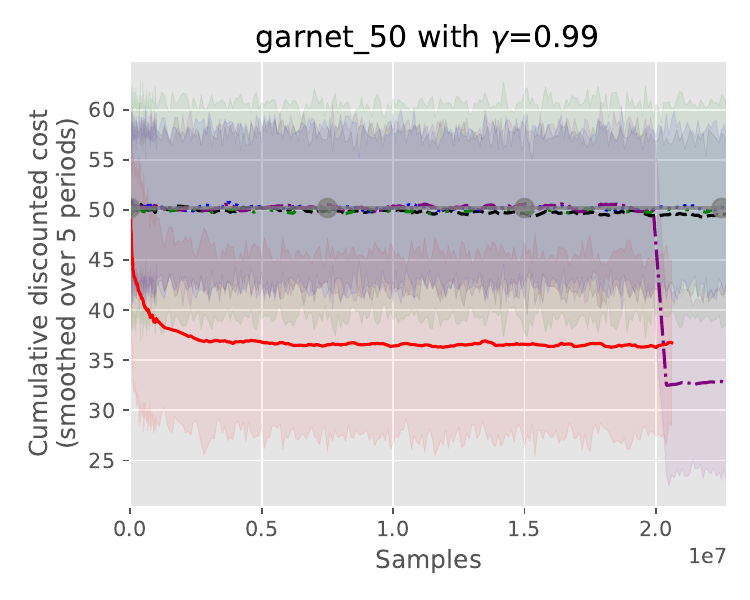}
    \end{subfigure}
    \\ %%
    \begin{subfigure}{.49\textwidth}
    \centering
    \includegraphics[width=.75\linewidth]{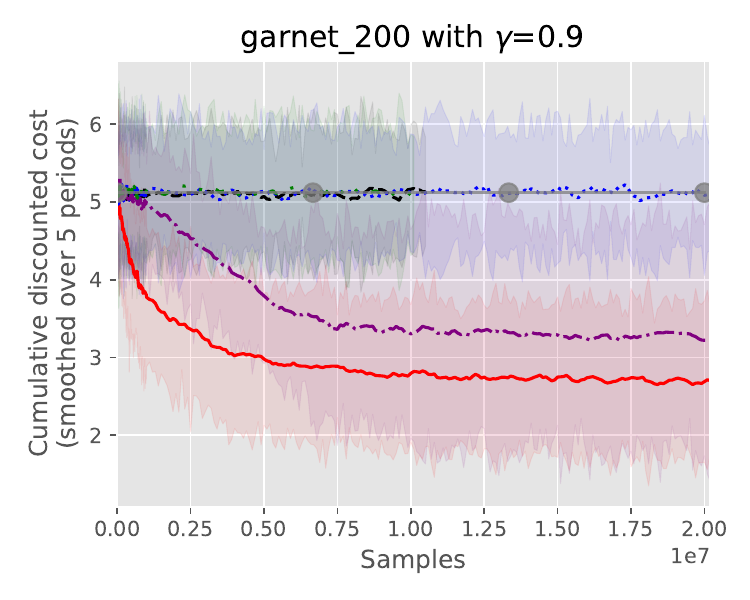}
    \end{subfigure}
    \begin{subfigure}{.49\textwidth}
    \centering
    \includegraphics[width=.75\linewidth]{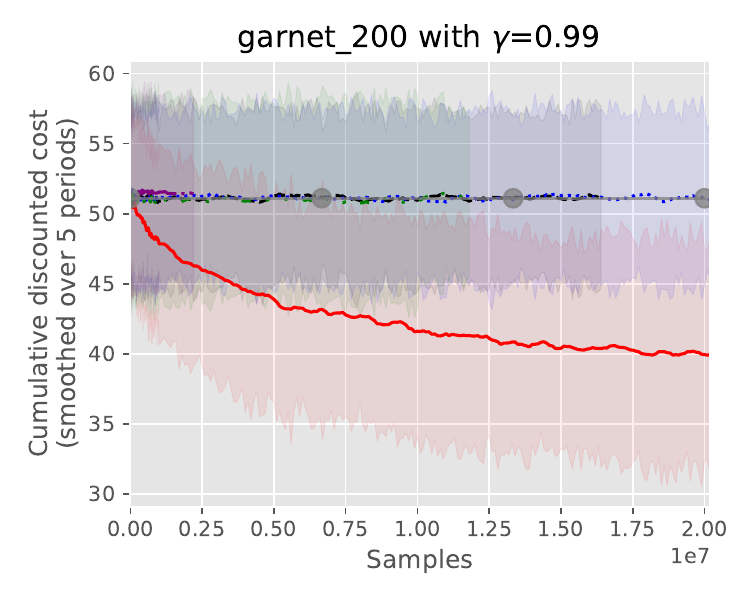}
    \end{subfigure}
    \caption{Estimated value function (lower is better) for GARNET. 
    The mean is shown in the line and the 95\% confidence interval in the shaded region over 10 seeds.
    The gray line with circles is the average baseline performance using a uniform policy.
    }
    \label{fig:garnet_online}
\end{figure}

The first environment is a random instance called GARNET~\cite{tarbouriech2021provably}, where we ran on a smaller ($\vert \cS \vert=50$ and $\vert \cA \vert=5$) and a larger ($\vert \cS \vert=200$ and $\vert \cA \vert=30$) instance.
As shown in Figure~\ref{fig:garnet_online}, SPMD MC-Dyn consistently outperforms all other methods.
In contrast, for garnet\_50 ($\gamma=0.99$), the best is SPMD+CTD-Dyn, which has a late but large improvement.
The reason for this behavior is that the best set of hyperparameters yields a high-accuracy policy evaluation (via a large sample count per CTD call), followed by a large policy update stepsize.
We tested the same hyperparameters on other problems but observed a more modest improvement.
Meanwhile, SPMD MC-Est’s performance degrades in larger environments or with larger discount factors.
Finally, both SPMD+CTD-Est and SARSA cannot improve beyond the baseline.

\begin{figure}
    \begin{subfigure}{.49\textwidth}
    \centering
    \includegraphics[width=.75\linewidth]{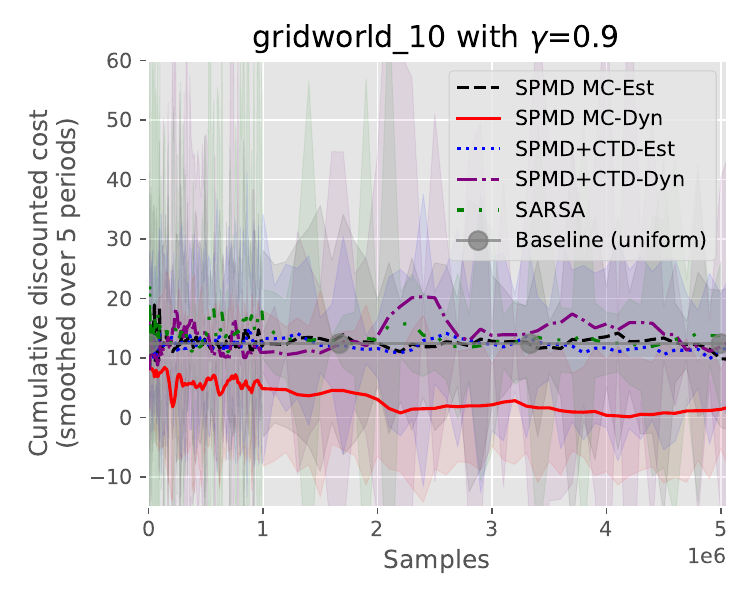}
    \end{subfigure}
    \begin{subfigure}{.49\textwidth}
    \centering
    \includegraphics[width=.75\linewidth]{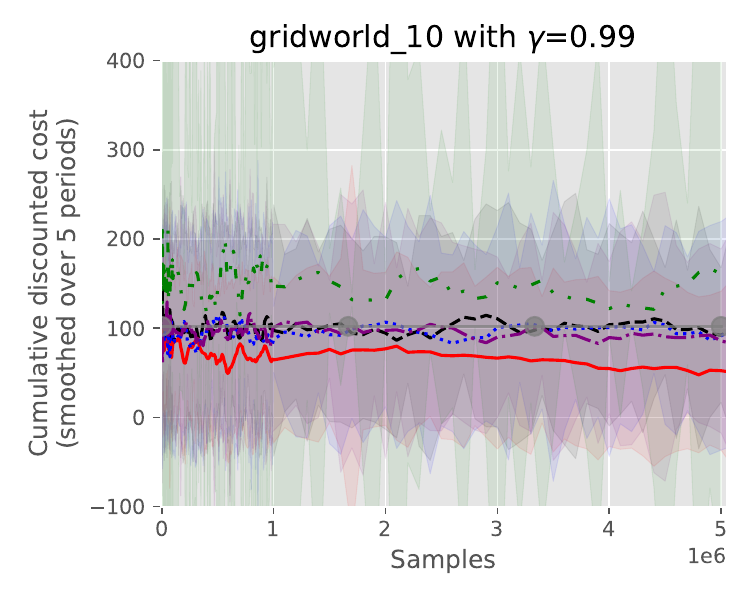}
    \end{subfigure}
    \\ %%
    \begin{subfigure}{.49\textwidth}
    \centering
    \includegraphics[width=.75\linewidth]{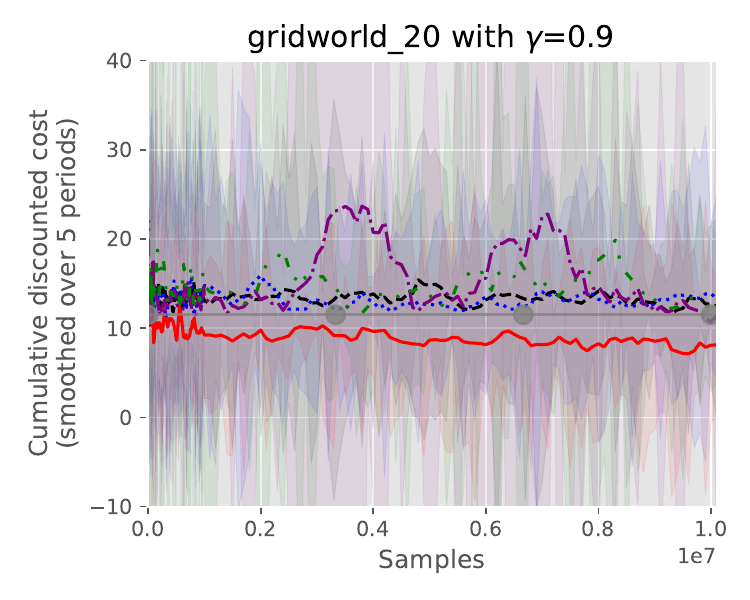}
    \end{subfigure}
    \begin{subfigure}{.49\textwidth}
    \centering
    \includegraphics[width=.75\linewidth]{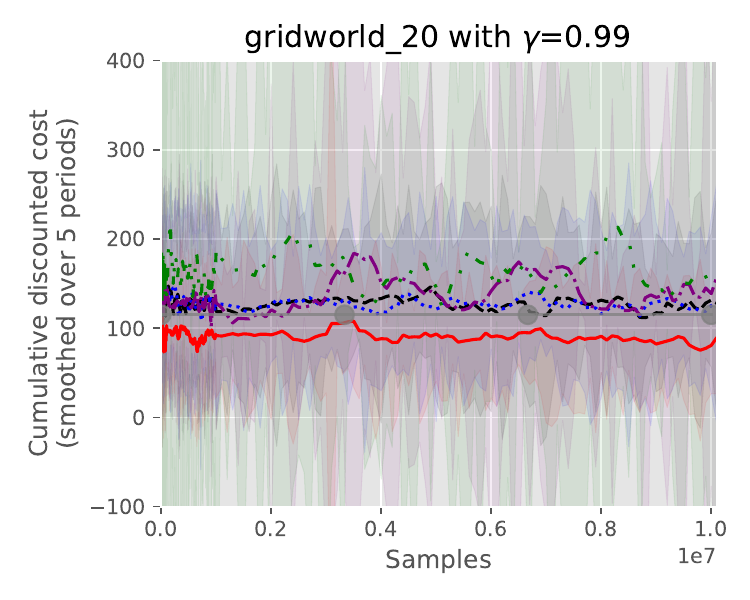}
    \end{subfigure}
    \caption{Estimated value function (lower is better) for GridWorld. Same setup as Figure~\ref{fig:garnet_online}.}
    \label{fig:gridworld_online}
\end{figure}

The second environment is a 2D route planning problem called GridWorld~\cite{dann2014policy}.
An agent moves in one of $\vert \cA \vert=4$ cardinal directions (with a 5\% chance of a uniform random action) from a random origin state to a target while avoiding traps.
We ran on a smaller problem of length $10$ ($\vert \cS \vert=100$) and a larger one of length $20$ ($\vert \cS \vert=400$).
As seen in Figure~\ref{fig:gridworld_online}, SPMD MC-Dyn is consistently the best performer.
On the other hand, SPMD+CTD-Dyn does not perform well here.
This may be because the problem lacks a low-dimensional structure that a linear function approximator seeks to exploit, resulting in slower convergence. 
Indeed, a similar poor performance using linear function approximation compared to nonlinear ones was also observed in~\cite{ju2022policy}.

\begin{table}
    \centering
    \caption{Per iteration cost of SPMD}
    \label{tab:samples_per_iter}
    \begin{tabular}{@{}lrrrr@{}} \toprule
    Env ($\gamma=0.9$) & MC-Dyn & MC-Est & CTD-Dyn & CTD-Est \\ \midrule 
    gridworld\_10 & $3.5 \times 10^4$ ($2.4$s) & $1.7 \times 10^7$ ($665$s) & $1.4 \times 10^5$ ($5.8$s) & $4.2 \times 10^7$ ($1885$s) \\
    gridworld\_20 & $4.0 \times 10^6$ ($223$s) & - (-) & $1.4 \times 10^6$ ($68$s) & - (-) \\
    garnet\_200 & $2.4 \times 10^4$ ($1.3$s) & $8.2 \times 10^5$ ($41$s) & $3.2 \times 10^4$ ($1.5$s) & $8.4 \times 10^6$ ($258$s) \\
    \bottomrule
    \end{tabular}
    \caption*{Average samples (runtime in seconds) per SPMD iteration with different policy evaluations.
    Both estimation-based approaches failed to finish one iteration in gridworld\_20 after $5 \times 10^7$ samples.}
\end{table}

The efficient performance of the dynamic exploration time is due to avoiding a conservative trajectory length.
As seen in Table~\ref{tab:samples_per_iter}, the dynamic approach uses fewer samples per policy evaluation (34x to 485x fewer samples than the estimation approach), yielding more policy updates.
The scale of the runtime reduction (run on one core of an Intel Xeon 6972P processor with 8GB of memory) matches the reduction in sample complexity, confirming that the dynamic exploration time incurs minimal computational overhead.
To close, we note that CTD-Dyn uses fewer samples per update than MC-Dyn in gridworld\_20, but its performance is not better (Figure~\ref{fig:gridworld_online}).
This is likely because CTD-Dyn cannot efficiently make the bias in~\eqref{eq:weight_diff_bias} small due to a lack of low-dimensional features.

\subsection{Performance on environments without state exploration}
\begin{figure}[t]
    \begin{subfigure}{.49\textwidth}
    \centering
    \includegraphics[width=.75\linewidth]{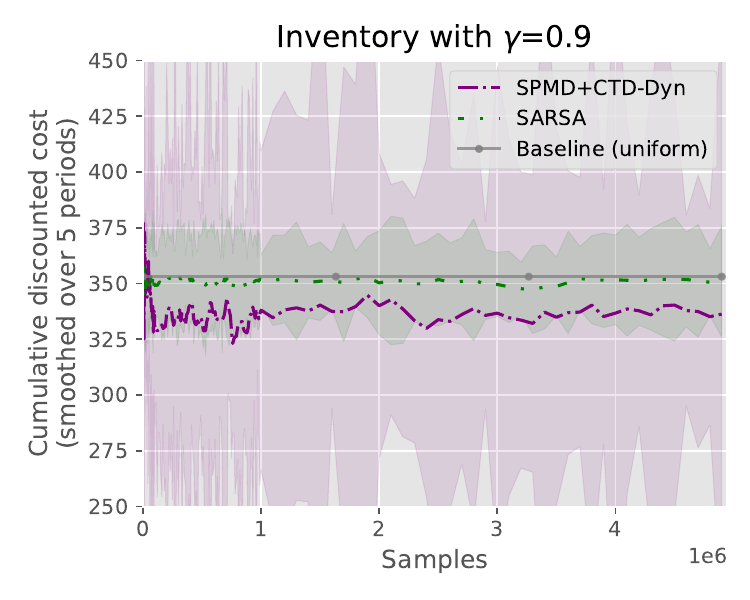}
    \end{subfigure}
    \begin{subfigure}{.49\textwidth}
    \centering
    \includegraphics[width=.75\linewidth]{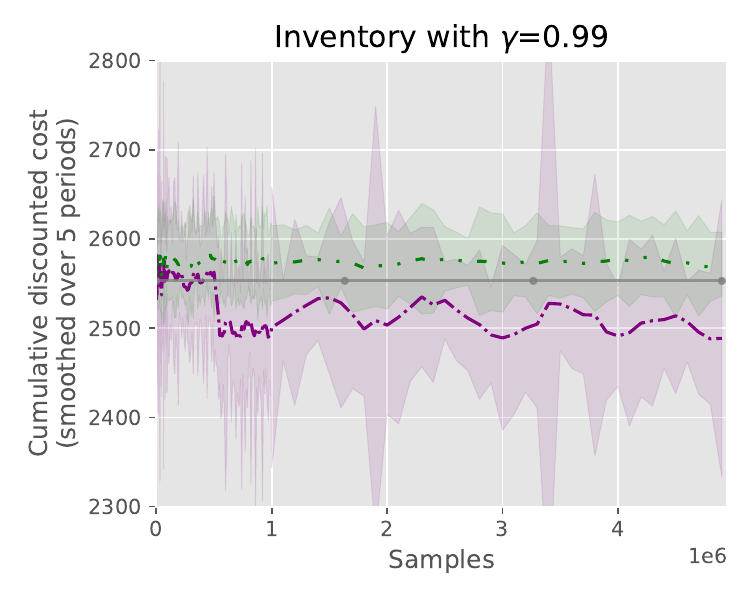}
    \end{subfigure}
    \\ %%
    \begin{subfigure}{.49\textwidth}
    \centering
    \includegraphics[width=.75\linewidth]{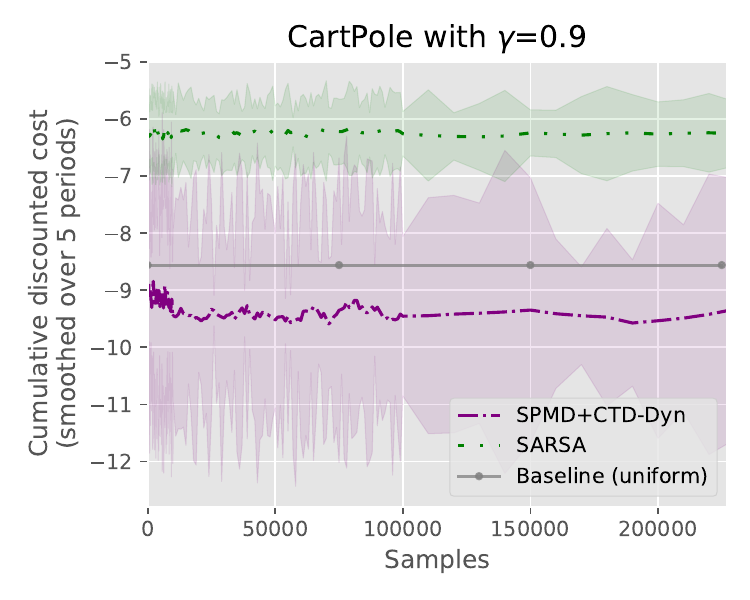}
    \end{subfigure}
    \begin{subfigure}{.49\textwidth}
    \centering
    \includegraphics[width=.75\linewidth]{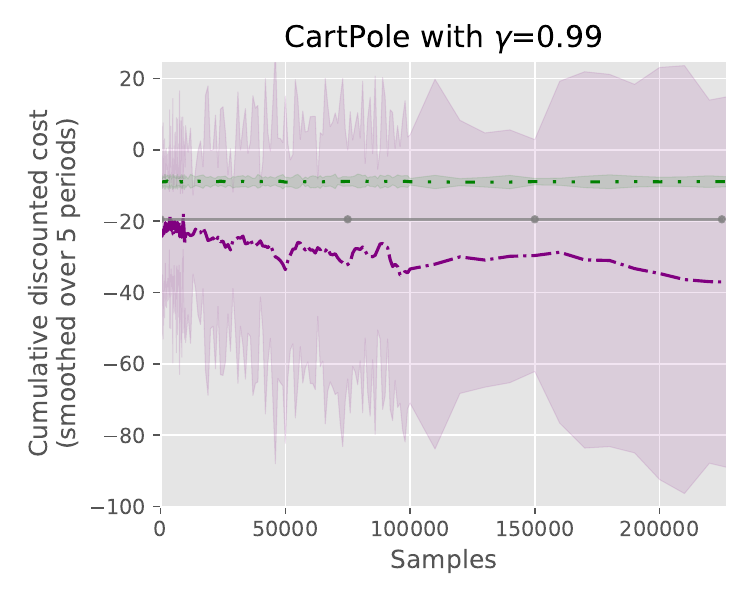}
    \end{subfigure}
    \caption{Estimated value function (lower is better) for Inventory and CartPole. 
    Same setup as Table~\ref{fig:garnet_online}.
    Methods that cannot finish one policy evaluation step can only match the baseline and are omitted.}
    \label{fig:gym_online}
\end{figure}

Unlike the previous problems, we now consider environments without state exploration.
The first is a multi-period inventory (newsvendor) problem with a single product subject to ordering, holding, and backlog costs as well as random demand~\cite{shapiro2021lectures}.
Because the problem has continuous states and actions, we discretized them to size $\vert \cS \vert = 81$ and $\vert \cA \vert= 10$.
The second is a classical robotics problem, CartPole, where a pole is to be balanced on a cart that is pushed left or right ($\vert \cA \vert=2$)~\cite{michie1968boxes}.
For this continuous state problem, all $n=4$ dimensions are discretized to size 20, yielding $\vert \cS \vert=20^4=160000$ states.
Although these problems lack convergence guarantees since Assumption~\ref{asmp:optimal_mixing} may be violated, our goal is simply to stress-test the methods on them.

Figure~\ref{fig:gym_online} shows the performance in these two problems, where visiting every state is practically infeasible.
SPMD+CTD-Dyn is the best-performing method on average. 
On the other hand, SARSA only matches or performs worse than the baseline.
The reason for the latter is that it under-explores when the $\epsilon$-exploration parameter is not sufficiently large, leading to over-exploitation and sub-optimal actions.
Both SPMD MC-Dyn and the estimation methods can only match the baseline performance because they cannot visit every state, which is required by their policy evaluation step.
For example, in Inventory, 3 states (corresponding to empty inventory) were unvisited after $5 \times 10^6$ samples. 
In CartPole, 155040 states were unvisited after $2.5 \times 10^5$ samples.
Conversely, SPMD+CTD-Dyn can update the policy without visiting every state due to the sampling distribution in~\eqref{eq:w_dist_defn}.

\section{Conclusion} \label{sec:conclusion}
We introduce a class of auto-explorative SPMD methods with the first algorithm-independent $\tilde{O}(\epsilon^{-2})$ sample complexity to solve online RL problems in a parameter-free, on-policy manner. 
We provide two dynamic exploration times: one for the tabular setting and one for the function approximation setting.
In the latter, we propose a new sampling distribution based on the discounted visitation distribution, which may be of independent interest since it covers a more general class of MDPs, is convenient to sample from, and has encouraging numerical performance.

To conclude, we discuss some extensions.
First, while auto-exploration can be applied to long-term average costs by solving a sequence of discounted problems~\cite{jin2021towards}, this will lead to sub-optimal dependence on $\epsilon$.
Improving this dependence is an interesting research question.
Second, it is interesting to consider how to apply auto-exploration with function approximation to general state space problems.
Third, we ask if one can still achieve auto-exploration over arbitrary Markov chains, even for those that lack the ergodicity assumption on the optimal policy (Assumption~\ref{asmp:optimal_mixing}).

\appendix

\section{Missing details from Section~\ref{sec:basics}} \label{sec:last_iterate}
Our goal is to prove the last-iterate convergence, that is, Proposition~\ref{prop:stronger_gap_converge_with_relaxed}.
The result had two parts.
The second part~\eqref{eq:dist_to_opt_at_t}, was already shown in~\cite{li2025policy}.
Therefore, we set out to prove the first part, which is~\eqref{eq:last_iter_opt_gap}.

For any time interval $[t_0,t_1]$, consider the sequence $\{L_t\}_{t=t_0}^{t_1}$ given recursively by
\begin{talign*} 
  L_t := \begin{cases}
    e^{-1} \cdot (t_1-t_0+1)^{-1} &: t=t_1 \\
    L_{t+1} + L_{t+1}^2 &: t_0 \leq t < t_1.
  \end{cases}
\end{talign*}
This sequence is identical to the one given in~\cite[Eq'n~(3.2)]{jain2021making}, and we omit its dependence of $t_0$,$t_1$ since it will be clear from context or explicitly provided.
Next, we define the centered optimality gap for a group of iterates.
For any constant $\varsigma \geq 0$, any state $s \in \cS$, and $l \in [t_0,t_1]$,
\begin{talign} \label{eq:A_func_def}
  &A(s,l,t_1) :=  \varsigma + \sum_{t=l}^{t_1} L_t \cdot (
    \\ & (1-\gamma)\eta_t[V^{\pi_t}(s) - V^{\pi_l}(s)] - \eta_t^2 \cdot (\hat{Q}^2 + M_h^2)). \nonumber
\end{talign}
The following is the key lemma for last-iterate convergence. 
It is similar to~\cite[Lemma 3.2]{jain2021making} but with two key differences: $\tilde{Q}^{\pi_t}$ is biased, and there is a separate objective value at each state $s \in \cS$.
Let us recall $\xi_t$ as the random vector used to generate the stochastic estimator $\tilde{Q}^{\pi_t}$. Denote the collection of random vectors $\xi_{[t-1]} := \{\xi_0,\xi_1,\ldots,\xi_{t-1}\}$ as the complete history up to time $t-1$ (inclusively). 
\begin{lemma} \label{lem:weighted_deviation}
  Let $p$ be any distribution over $\{t_0,\ldots,t_1\}$.
  If $\eta_t$ is a non-increasing stepsize and both~\eqref{eq:second_moment} and~\eqref{eq:weight_diff_bias} hold, then for any state $s \in \cS$ and $\theta \geq 0$, we have
    $\Pr\{ \sum_{l=t_0}^{t_1} p(l) \cdot A(s,l,t_1) > \theta + 2\eta_{t_0} \varsigma \} \leq \exp\{ -\frac{\theta}{8\eta_{t_0}^2 \hat{Q}^2} \}$.
\end{lemma}
\begin{proof}{Proof}
  For any state $q \in \cS$, define the bias and distance to optimality metric, respectively, 
      $\Delta_t(q) := \langle -[\tilde Q^{\pi_t}(q, \cdot) - Q^{\pi_t}(q,\cdot)], \pi_t(\cdot \vert q) - \pi_{t_0}(\cdot \vert q)\rangle$
      and
      $d_t(q) := \|\pi_t(\cdot \vert q) - \pi_{t_0}(\cdot \vert q)\|$.
  In view of~\eqref{eq:second_moment} and~\eqref{eq:weight_diff_bias}, one can show by the Cauchy-Schwarz inequality, $\mathbb{E}[\Delta_t(q) \vert \xi_{[t-1]}] \leq \varsigma \cdot d_t(q)$ and $\vert \Delta_t(q) \vert \leq 2\hat{Q}d_t(q)$.
  Then by applying Hoeffding's lemma, we have for any $\mu > 0$,
  \begin{talign*}
    &\mathbb{E}[\exp\{\mu \Delta_t(q)\} \vert \xi_{[t-1]}] 
    \\ &\leq \exp \{\mu d_t(q) \varsigma+ 2\mu^2 \hat{Q}^2 d_t(q)^2 \}.
  \end{talign*}
  \sloppy Fix $l \in [t_0,t_1]$ and define the exponentiated sums,
    $M_t(q) := \exp\{ - 2\lambda\eta_{t_0} \varsigma + \sum_{j=l}^t -2\lambda L_j \eta_j \Delta_j(q) - \lambda L_j^2 \cdot d_j(q)^2 \}$
    and
    $M_t(s,\pi) := \exp\{ - 2\lambda\eta_{t_0} \varsigma + \sum_{j=l}^t \mathbb{E}_{q \sim \kappa^{\pi_l}_s}[-2\lambda L_j \eta_j \Delta_j(q) - \lambda L_j^2 \cdot d_j(q)^2 ]\}$.
  Now, fixing the constant $\lambda = (8\eta_{t_0}^2 \hat{Q}^2)^{-1}$, one can show
  \begin{talign*}
    &\mathbb{E}[M_t(q) \vert \xi_{[t-1]}]  
    \\ &\leq
    M_{t-1}(q) \exp\{-\lambda L_t^2 d_t(q)^2 + 8\lambda^2L_t^2 \eta_t^2 \hat{Q}^2 d_t(q)^2 
    \\ & + 2\lambda L_t \eta_t \cdot d_t(q) \varsigma\} 
    \\ &\leq
    M_{t-1}(q) \exp\{2\lambda L_t \eta_t \cdot \varsigma\},
  \end{talign*}
  where the first inequality follows similarly to the proof of~\cite[Lemma 3.2]{jain2021making}, except there is an additional term for the bias, and the second inequality is by choice of $\lambda$ and the simple bound $d_t(q) \leq 2$.
  Using the fact $M_{l}(q) = \exp\{-2\eta_{l}\varsigma\}$, then recursively applying the above produces
  \begin{talign*}
    &\mathbb{E}[M_{t_1}(s,\pi)] 
    \\ &\leq
    \mathbb{E}_{q \sim \kappa^{\pi_l}_s} \mathbb{E}[M_{t_1}(q)]  
    \\ &\leq
    \exp\{-2\lambda \eta_{t_0}\varsigma\} \cdot \exp\{2\varsigma\lambda \sum_{t=l+1}^{t_1} L_t \eta_t\} 
    \\ &\leq
    \exp\{-2\lambda \eta_{t_0}\varsigma\} \cdot \exp\{2\varsigma\lambda \eta_{t_0}\} = 1, 
  \end{talign*}
  where the first inequality is by Jensen's inequality and the boundedness of $M_t(q)$ to exchange expectations, and the last inequality is by the fact $L_t \leq (t_1-t_0+1)^{-1}$ and the stepsize $\eta_t$ is non-increasing.

  Additionally, by combining the optimality conditions for the SPMD update~\eqref{eq:spmd_update} and Lemma~\ref{lem:performance_diff_deter} (see~\cite[Lemma 4]{lan2023policy} for a detailed derivation),
  \begin{talign*}
    &\eta_t \cdot (1-\gamma) \cdot [V^{\pi_t}(s) - V^{\pi_{l}}(s)]
    \\ &\leq
    \mathbb{E}_{q \sim \kappa^{\pi_{l}}_s}[D^{\pi_{l}}_{\pi_t}(q) - D^{\pi_{l}}_{\pi_{t+1}}(q) 
    \\ &+ 
    \eta_t^2(\hat{Q}^2 + M_h^2) - \eta_t\Delta_t(s)].
  \end{talign*}
  Multiplying both sides by $L_t$ and taking a telescopic sum,
  \begin{talign*}
    &\sum_{t=l}^{t_1} L_t\eta_t \cdot (1-\gamma) \cdot [V^{\pi_t}(s) - V^{\pi_{l}}(s)] \\
    &\leq
    \sum_{t=l}^{t_1} L_t \cdot \mathbb{E}_{q \sim \kappa^{\pi_{l}}_s}[D^{\pi_{l}}_{\pi_t}(q) - D^{\pi_{l}}_{\pi_{t+1}}(q) 
    \\& + \eta_t^2(\hat{Q}^2 + M_h^2) + \eta_t\Delta_t(s)] \\
    &\leq
    \sum_{t=l+1}^{t_1} [(L_t - L_{t-1}) \mathbb{E}_{q \sim \kappa^{\pi_{l}}_s}D^{\pi_{l}}_{\pi_t}(q) 
    \\ &+ L_t \cdot \eta_t^2(\hat{Q}^2 + M_h^2) - \mathbb{E}_{q \sim \kappa^{\pi_{l}}_s}L_t \cdot \eta_t\Delta_t(s)] \\
    &\leq
    \sum_{t=l+1}^{t_1} [-L_t^2 \cdot \mathbb{E}_{q \sim \kappa^{\pi_{l}}_s}d_t(q) + L_t \cdot \eta_t^2(\hat{Q}^2 + M_h^2) 
    \\ & - \mathbb{E}_{q \sim \kappa^{\pi_{l}}_s}L_t \cdot \eta_t\Delta_t(s)],
  \end{talign*}
  where the last line follows from $L_t-L_{t-1} = -L_t^2$ (from its definition) and the 1-strong convexity of the Bregman divergence.
  The rest of the proof follows similarly to~\cite[Lemma 3.2]{jain2021making}, so we skip it.
\end{proof}

\noindent The sum $A(s,l,t_1)$ from~\eqref{eq:A_func_def} depends on the group of policies $\pi_t$ from iterations $l,\ldots,t_1$ for an $l \geq t_0$. It remains to define the indicies $t_0,t_1$. 
Given $k \geq 1$ iterations, we form $E+1$ groups of iterates separated at checkpoint iterates $\ki{i}$ (identical to $T_i$ in~\cite[Equation (2.2)]{jain2021making}), where
$E := \lceil \log(k) \rceil$ and
\begin{talign*}
  \ki{i} &:= \begin{cases}
      k - \lceil k \cdot 2^{-i} \rceil &: 0 \leq i \leq E \\
      k &: i = E+1
  \end{cases}.
\end{talign*}
We refer to the iterates in $[\ki{i}+1,\ki{i+1}]$ as the $i$-th group of iterates.
These groups halve in size with increasing $i$, as seen by $\ki{i} - \ki{i-1} \approx k \cdot 2^{-i}$ for epoch $i \geq 1$. So the first group of iterates is the largest. Note that because $k \geq 4$, then $\ki{1} = k - \lceil k/2 \rceil \geq k/4+1$.

It turns out this group of iterates shares a similar objective function, as shown below.
We skip the proof, since it is nearly identical to~\cite[Lemma 3.6]{jain2021making}. The main technical difference is that we apply the previous Lemma~\ref{lem:weighted_deviation}, leading to an additional bias term $\varsigma$.
\begin{lemma} \label{lem:epoch_descent}
    For epoch $i \geq 0$, let $q^{(i)}$ be an arbitrary distribution over $\{\ki{i},\ldots,\ki{i+1}\}$.
    Define the probability distribution over $\{\ki{i+1}+1,\ki{i+2}\}$: $p^{(i+1)}(t) := \Psi^{-1} \cdot L_t\eta_t$, where $\Psi := \sum_{t=\ki{i+1}+1}^{\ki{i+2}} L_t \eta_t$ and $L_t$ is defined w.r.t.~$t_0=\ki{i}+1$ and $t_1=\ki{i+2}$. 
    Then for any $\delta \in (0,\frac{1}{2}]$, $s \in \cS$, and non-increasing stepsize $\eta_t$,
    \begin{talign*}
        &\Pr\{ \sum_{t=\ki{i+1}}^{\ki{i+2}-1} p^{(i+1)}(t) \cdot  V^{\pi_t}(s)  \\
        &\leq
        \sum_{t=\ki{i}}^{\ki{i+1}-1} q^{(i)}(t) \cdot V^{\pi_t}(s) 
        \\ &+ E_3(\varsigma,\delta,\eta_{\ki{i+2}}, \eta_{\ki{i}})
        \} \leq \delta,
    \end{talign*}
    where $E_3(\varsigma,\delta,\eta',\eta) := \frac{5e}{2(1-\gamma) \eta'}(9\eta^2 \hat{Q}^2\log\frac{1}{\delta} + 2\eta\varsigma)$.
\end{lemma}

The final supporting result we need is the average optimality gap for the first group, which is average-iterate convergence for the first $\ki{1} - \ki{0} \geq k/4$ iterates. The proof is skipped since it is nearly identical to the one found in~\cite[Theorem 5.3]{li2025policy}.
\begin{proposition} \label{prop:ergodic_iterate_direct}
    Under the same setup as Proposition~\ref{prop:stronger_gap_converge_with_relaxed}, then
    \begin{talign*}
        \mathrm{Pr}\big \{\exists s \in \cS ~:~
        &\frac{1}{\ki{1}-\ki{0}} \sum_{t=\ki{0}}^{\ki{1}-1}[V^{\pi_t}(s) - V^{\pi^*}]
        \\ &\leq
        E_4(k,\varsigma,\delta)
        \big \}
        \leq
        \delta
    \end{talign*}
    where $E_4(k,\varsigma,\delta) := \frac{4\alpha^{-1} \bar{D}_0 + \alpha (\hat{Q}^2 + M_h^2) + 8\hat{Q}\sqrt{\log(\vert \cS \vert/\delta)}}{(1-\gamma) \sqrt{k}} + \frac{4\varsigma}{1-\gamma}$.
    % {\bf GL: the summation of the left hand side should be from $\ki{0}$ to $\ki{1}$?}
\end{proposition}

We are now ready to prove the main result of this section.
\begin{proof}[Proof of~\eqref{eq:last_iter_opt_gap} from Proposition~\ref{prop:stronger_gap_converge_with_relaxed}]
    Choose $\delta = \delta'/(2(E+1) \vert \cS \vert)$ in Lemma~\ref{lem:epoch_descent} for all states $s \in \cS$ and epochs $i=0,\ldots,E = \lceil \log(k) \rceil$.
    By union bound, the probability the events hold over all epochs $i$ and states $s \in \cS$ is $\delta'/2$ . 
    Therefore, for the remainder of the proof we assume
    \begin{talign*}
        &\sum_{t=\ki{i+1}+1}^{\ki{i+2}} p^{(i+1)}(t) V^{\pi_t}(s)  
        \\ &\leq
        \sum_{t=\ki{i}+1}^{\ki{i+1}} q^{(i)}(t) \cdot V^{\pi_t}(s) 
        \\& + \frac{5e}{2(1-\gamma)}(9\eta \hat{Q}^2\log\frac{2(E+1) \vert \cS \vert}{\delta'} + 2\varsigma)
    \end{talign*}
    holds for all epochs $i$ and states $s$.
    Recall $q^{(i)}$ is a distribution with support over $\{\ki{i},\ldots,\ki{i+1}\}$ that can be arbitrarily defined.

    Recursively applying the above bound, and setting $q^{(0)}(\cdot) = (\ki{1}-\ki{0})^{-1}$ and $q^{(i)} = p^{(i-1)}$ for $i \geq 1$, we arrive at
    \begin{talign*}
        &\sum_{t=\ki{E}}^{\ki{E+1}} p^{(E)}(t) V^{\pi_t}(s) 
        \\ &\leq
        (\ki{1}-\ki{0})^{-1} \sum_{t=\ki{0}+1}^{\ki{1}} V^{\pi_t}(s) 
        \\ &+ 
        \frac{5(E+1)e}{2(1-\gamma)} (9\eta \hat{Q}^2 \log \frac{2(E+1)\vert \cS \vert}{\delta'} + 2 \varsigma),
    \end{talign*}
    which holds for all $s \in \cS$ simultaneously.

    We can simplify the above with the following two observations.
    First, recall $\ki{E} = k-1$ and $\ki{E+1} = k$. And since $p^{(E)}(t)$ is a probability distribution over $\{\ki{E}+1,\ki{E+1}\}$, then $p^{(E)}(\ki{E}) = 1$, so the left hand side is equal to $V^{\pi_k}(s)$.
    Second, $E+1 \leq \log(k)+2 = \log(4k)$.
    After adding $-V^{\pi^*}(s)$ to both sides, recalling $\ki{1}-\ki{0} \geq k/4$, and noting our choice in $\eta$, we get
    \begin{talign*}
        &V^{\pi_k}(s) - V^{\pi^*}(s)  
        % \\
        % &\leq
        % (\ki{1}-\ki{0})^{-1} \sum_{t=\ki{0}+1}^{\ki{1}} [V^{\pi_t}(s) - V^{\pi^*}(s)] + \frac{5e\log(4k) }{2(1-\gamma)} \cdot (9\eta \hat{Q}^2 \log \frac{2\vert \cS \vert \log(4k)}{\delta'} + 2 \varsigma)
        \\
        &\leq
        \frac{4\alpha^{-1} \bar{D}_0 + \alpha (\hat{Q}^2 + M_h^2) + 8\hat{Q}\sqrt{\log(2\vert \cS \vert/\delta')}}{(1-\gamma) \sqrt{k}} + \frac{4\varsigma}{1-\gamma}
        \\ &+ 
        \frac{5e\log(4k) }{2(1-\gamma)} \cdot (\frac{9\alpha \hat{Q}^2\log \frac{2\vert \cS \vert \log(4k)}{\delta'} }{\sqrt k} + 2 \varsigma)
        % \\
        % &\leq
        % \frac{4\alpha^{-1} [\bar{D}_0^2+1] + \alpha(M_h^2 + \hat{Q}^2[1 + 4\log(2\vert \cS \vert/\delta') + \frac{45e}{2}\log(4k)\log(\frac{2 \vert \cS \vert\log(4k)}{\delta'})])}{(1-\gamma)} 
        % + \frac{\varsigma(4 + 5e\log(4k))}{1-\gamma}
        % \\
        % &\leq
        % \frac{4\alpha^{-1} [\bar{D}_0^2+1] + 67\alpha(\hat{Q}^2 + M_h^2)\log(4k)\log(\frac{2\vert \cS \vert\log(4k)}{\delta'})}{(1-\gamma)\sqrt{k}} 
        % + \frac{18\varsigma\log(4k)}{1-\gamma}
        \\
        &\leq
        \frac{4\alpha^{-1} [\bar{D}_0+1] + 67\alpha(\hat{Q}^2 + M_h^2)(\log \frac{4\vert \cS \vert}{\delta'})^2}{(1-\gamma)\sqrt{k}} 
        + \frac{18\varsigma\log(4k)}{1-\gamma}
    \end{talign*}
    where the first inequality holds with probability $1-\delta'/2$ (Proposition~\ref{prop:ergodic_iterate_direct} with $\delta = \delta'/2$), and the second inequality is by Young's inequality on the square root term and simplifying terms.
    Finally, by union bound, the total probability of success $1-2(\delta'/2) = 1-\delta'$.
\end{proof}

\section{Missing proof from Section~\ref{sec:implicit}} \label{sec:proofs_for_yanli_work}
\begin{proof}[Proof of Lemma~\ref{lem:tsallis_strong}]
    For notational convenience, define $D := \mathrm{rint}\Delta_{\vert \cA \vert}$.
    By Taylor's theorem~\cite[Theorem 2.1]{nocedal2006numerical}, it suffices to find a $\mu_\omega$ such that
        $\inf_{u \in D,~\|x\|_1=1} \{x^T \nabla^2 \omega(u) x \}
        \geq 
        \mu_\omega$.
    To show $\mu_\omega=1$ satisfies this inequality, observe that the Hessian of $\omega(u) = -\frac{1}{(1-p)p} \sum_{a \in \cA} [u(a)]^{p}$ at any point $u \in \mathrm{rint}\Delta_{\vert \cA \vert}$ is a diagonal matrix with entries $[\nabla^2 \omega(u)]_{ii} = \frac{(1-p)p}{(1-p)p} u_i^{p-2}$. 
    Therefore,
    \begin{talign*}
      &\inf_{u \in D, \|x\|_1=1} \{x^T\nabla^2 \omega(u) x\}
      % &=
      % \inf_{u \in D, \|x\|_1=1} \{ \sum_{i=1}^{\vert \cA \vert} x_i^2 u_i^{p-2}\}
      % \\
      \\ &=
      \inf_{u \in D} \min_{x \geq 0} \big\{ \sum_{i=1}^{\vert \cA \vert} x_i^2 u_i^{p-2} : \sum_{i=1}^{\vert \cA \vert} x_i = 1\big\}
      \\
      &=
      \inf_{u \in D} \big( \sum_{i=1}^{\vert \cA \vert} u_i^{2-p} \big)^{-1}
      \\
      &=
      \big( \sup_{u \in D} \sum_{i=1}^{\vert \cA \vert} u_i^{2-p} \big)^{-1}
      \\
      &\geq
      \big( \max_{u \in \Delta_{\vert \cA \vert}} \sum_{i=1}^{\vert \cA \vert} u_i^{2-p} \big)^{-1}
      =
      1,
    \end{talign*}
    % ({\bf GL: Please use a more general $p$ in the above derivation.})
    \sloppy where the second line is because $x_i^* = u_i^{2-p}/(\sum_{i=1}^{\vert \cA \vert} u_i^{2-p})$ is optimal (by the KKT conditions), and the last line is the maximizer of a convex function over the simplex lies at an extreme point.
\end{proof}

\section{Missing details from Section~\ref{sec:spmd_linear}} 
Our results for Section~\ref{sec:spmd_linear} rest around the following nearly non-expansive result. 
Recall $W = \mathrm{diag}(w)$, where $w(z) = \kappa^{\piexplst}_{\sor}(s) \cdot \piexplact(a \vert s)$ from~\eqref{eq:w_dist_defn}, and $P^\pi((s,a),(s',a')) = P(s' \vert s,a) \pi(a' \vert s')$.
\begin{lemma} \label{lem:nearly_nonexpand}
    If $\epspi \in (0,\frac{1-\gamma}{2}]$, then $\|P^\pi u\|_W^2 \leq \gamma^{-3/2} \|u\|_W^2$ for all $u \in \mathbb{R}^{\vert \cZ \vert}$.
\end{lemma}
\begin{proof}
    Throughout the proof, we denote $z = (s,a)$ and $z' = (s',a')$. 
    Recall the $t$-step distribution under $\pi$, $P^{(t)}_\pi(\cdot,\cdot)$, defined below~\eqref{eq:visitation_measure}. 
    For any state $q \in \cS$,
    \begin{talign*} 
        &\sum_{z \in \cZ} [\kappa^{\piexpl}_{q}(s) \piexpl(a \vert s)] P^\pi(z,z')
        \\
        &=
        \sum_{z \in \cZ} [\kappa^{\piexpl}_{q}(s) \piexpl(a \vert s) P(s' \vert s,a)] \pi(a' \vert s')
        \\
        % &=
        % (1-\gamma) \sum_{t=0}^\infty \gamma^t \big[ \sum_{s \in \cS} P^{(t)}_\pi(q,s) [\sum_{a \in \cS} \pi(a \vert s) P(s' \vert s,a)]\big] \cdot \pi(a' \vert s') 
        % \\
        &=
        (1-\gamma) \sum_{t=0}^\infty \gamma^t P^{(t+1)}_{\piexpl}(q,s') \pi(a' \vert s') 
        \\
        &\leq
        \gamma^{-1} \cdot (1-\gamma) \sum_{t=0}^\infty \gamma^t P^{(t)}_{\piexpl}(q,s') \pi(a' \vert s') 
        \\
        &\leq
        \gamma^{-3/2} \cdot \kappa^{\piexpl}_q(s') \piexpl(a' \vert s'), \quad \forall z,z \in \cZ,
        % TODO: change time_shift_2 -> time_shift
    \end{talign*}
    where we used the definition of $\kappa^\pi_q(s) = (1-\gamma)\sum_{t=0}^\infty \gamma^t P_\pi^{(t)}(q,s)$, and the last line 
    is by $\pi(a' \vert s') \leq (1-\epspi)^{-1} \tpi(a' \vert s') \leq \gamma^{-1/2} \tpi(a' \vert s')$, which can be deduced from our construction $\piexplact(a \vert s) = (1-\epsact)\pi(a \vert s) + \frac{\epsact}{\vert \cA \vert}$ from~\eqref{eq:tpi_defn} and the assumption $\epspi \leq (1-\gamma)/2$.
    % Combining the last two observations,
    % \begin{talign*}
    %     &\sum_{z \in \cZ} [\kappa^{\pi}_{\sor}(s) \cdot \tpi(a \vert s)] P^\pi(z,z')
    %     \\
    %     &\leq
    %     \sum_{z \in \cZ} [\kappa^{\pi}_{\sor}(s) \cdot \pi(a \vert s)] P^\pi(z,z') 
    %     \\ &+ 
    %     \sum_{z \in \cZ} \frac{\kappa^{\pi}_{\sor}(s)}{\vert \cA \vert} \cdot \frac{(1-\gamma)\kappa^{\pi}_{\sor}(s') \cdot P^\pi(z,z')}{4}
    %     \\
    %     &\leq
    %     \gamma^{-1} \cdot \kappa^\pi_{\sor}(s')\pi(a' \vert s') + \frac{(1-\gamma) \kappa^\pi_{\sor}(s')\pi(a' \vert s')}{4} 
    %     \\
    %     &\leq
    %     (\gamma^{-5/4}) \cdot \gamma^{-1/4} \cdot \kappa^\pi_{\sor}(s') \cdot \tpi(a' \vert s'),
    % \end{talign*}
    % \sloppy where the first inequality is by~\eqref{eq:chipitpi_relationship}, the second is by~\eqref{eq:time_shift_2} and partly by $P^\pi(z,z') = P(s' \vert s,a) \pi(a' \vert s') \leq \pi(a' \vert s')$, and the last inequality is because $\gamma^{-1} + \frac{1-\gamma}{4} \leq \gamma^{-5/4}$ for $\gamma > 0$ and~\eqref{eq:pi_tpi_relationship}.
    From the above inequality and recalling $w(z) = \kappa^{\piexpl}_{\sor}(s) \cdot \piexpl(a \vert s)$, we conclude 
    \begin{talign*}
        &\|P^\pi u\|_W^2
        \\ &\leq
        \sum_{z \in \cZ} [\kappa^{\piexpl}_{\sor}(s) \tpi(a \vert s)] (\sum_{z' \in \cZ} P^\pi(z,z') u(z')^2)
        \\
        &=
        \sum_{z' \in \cZ} \big[\sum_{z \in \cZ} [\kappa^{\piexpl}_{\sor}(s) \tpi(a \vert s)] P^\pi(z,z')\big] u(z')^2
        \\
        &\leq
        \gamma^{-3/2} \sum_{z' \in \cZ} [\kappa^{\piexpl}_{\sor}(s') \tpi(a' \vert s')] u(z')^2
        \\ &=
        \gamma^{-3/2} \|u\|_W^2,
    \end{talign*}
    where the first inequality applied Jensen's inequality.
\end{proof}

\subsection{Missing proofs from subsection~\ref{sec:F_new}} \label{sec:missing_F_new_ec} % % \label{sec:missing_F_new}

\begin{proof}[Proof of Proposition~\ref{prop:proj_bellman_solvable}]
    \sloppy Note that the condition $\min_{s \in \cS} {\kappa}^{\tpi}_{\sor}(s) > 0$ and positivity of $\epsact$ ensure $\min_{z \in\cZ} w(z) > 0$. 
    In view of the assumption $\Phi$ is full column rank, then the projection operator $\Pi := \Phi[\Phi^TW\Phi]^{-1}\Phi^TW$ is well-defined. Here, $\Pi$ is defined s.t.~$\Pi u = \mathrm{argmin}_{v \in \mathrm{span}(\Phi)} \|u-v\|_W$.

    Proof for solvability of~\eqref{eq:proj_bellman_eqn_og} is similar to~\cite[Theorem 6.1.1]{puterman2014markov}, where we must show $\sigma(\gamma \Pi P^\pi) < 1$. Here, $\sigma(\cdot)$ is the spectral radius. Since $\|\cdot\|_W$ is a sub-multiplicative norm, we have
    \begin{talign*}
        \sigma(\gamma \Pi P^\pi)
        &\leq
        \gamma \|\Pi P^\pi\|_W
        \\ &\leq
        \gamma \|\Pi\|_W \|P^\pi\|_W
        \leq
        \gamma^{1/4} < 1,
    \end{talign*}
    where the penultimate inequality is since projection is non-expansive and Lemma~\ref{lem:nearly_nonexpand}.

    To bound the norm of the solution $\Qstr$, we use the fact the solution to~\eqref{eq:proj_bellman_eqn_og} satisfies the Neumann series, $\Qstr = (I - \gamma \Pi P^\pi)^{-1} \Pi c = \sum_{t=0}^\infty (\gamma \Pi P^\pi)^t \Pi c$.
    Therefore,
    \begin{talign*}
        \| \Qstr \|_W
        &\leq
        \sum_{t=0}^\infty \|\gamma \Pi P^\pi\|_W^t \|\Pi c\|_W
        \\ &\leq
        \sum_{t=0}^\infty \gamma^{t/4} 
        =
        \frac{1}{1-\gamma^{1/4}} \leq \frac{4}{1-\gamma},
    \end{talign*}
    where the second inequality is by our previous derivations and $\|c\|_W \leq \|c\|_\infty \leq 1$, since all costs are bounded by 1.
\end{proof}

\begin{proof}[Proof of Lemma~\ref{lem:monotone_with_disc_visit}]
    The proof is identical to~\cite[Lemma SM2.3]{li2023accelerated} combined with Lemma~\ref{lem:nearly_nonexpand}.
\end{proof}

\begin{proof}[Proof of Lemma~\ref{lem:F_lipschitz}]
    The proof is similar to~\cite[Lemma A.5]{li2024stochastic} combined with Lemma~\ref{lem:nearly_nonexpand}.
\end{proof}

\subsection{Missing proofs from subsection~\ref{sec:sto_hF_defn}} \label{sec:missing_sto_hF_defn}
\begin{proof}[Proof of Lemma~\ref{lem:F_bias}]
    Recall the sampled state $s_{t'}$ from the three-step construction of $\hF(\cdot)$ in subsection~\ref{sec:sto_hF_defn}, where the random time $t' := \tau^{\piexplst}(\sor) + \tilde t$ depends on the origin state's first hitting time under policy $\piexplst$, $\tau^{\piexplst}(\sor)$, and $\tilde t \sim \mathrm{Geo}(1-\gamma)$ is an iid random variable.
    Also recall that we only sample $s_{t'}$ if $\tilde t < m$ for a user-defined max trajectory length $m$.
    With this construction, observe that for any $s, \sor \in \cS$, the random state $s_{t'}$ is distributed according to (we write $s_{\tau^{\piexplst}(\sor)} = s_{t_0}$ for notational convenience),
    \begin{talign*}
        &\mathrm{Pr}\{s_{t'} = s \vert s_{\tau^{\piexplst}(\sor)} = \sor, \tilde t < m\}
        \\ &=
        \frac{1}{\mathrm{Pr}\{\tilde t < m\}} \sum_{t=0}^{m-1} \mathrm{Pr}\{\tilde t = t\} \mathrm{Pr}\{s_{t'}=s \vert s_{t_0} = \sor\}
        \\ &=
        \frac{1-\gamma}{\mathrm{Pr}\{\tilde t < m\}} \sum_{t=0}^{m-1} \gamma^t P^{(t)}_{\tpi}(\sor, s)
        \\
        &=:
        \kappa_{\sor, [m]}^{\tpi}(s), 
    \end{talign*}
    where the first equality is by independence of $\tilde t$ and the law of total probability, and the second is because the Markov chain is time homogeneous and we denote $P^{(t)}_{\tpi}(\sor,s) = \mathrm{Pr}\{s_t=s \vert s_{0} = \sor\}$.
    In view of $\hF(\theta) = \mathbf{0}$ when $\tilde t \geq m$, we can show for any $\theta \in \mathbb{R}^d$,
    \begin{talign*}
        &\mathbb{E}[\hF(\theta) \vert \tau^{\piexplst}(\sor)] 
        \\ &= 
        \mathrm{Pr}\{\tilde t < m\} \cdot \Phi^T W^{[m]}(\Phi \theta - c - \gamma P^{\tpi} \Phi \theta),
    \end{talign*} 
    where $w^{[m]}(z) := \kappa^{\tpi}_{\sor, [m]}(s) \cdot \piexplact(a \vert s)$ and $W^{[m]} := \mathrm{diag}(w^{[m]})$.
    With the state-action distribution $w(z) = \kappa_{\sor}^{\tpi}(s) \cdot \piexplact(a \vert s)$, we can confirm that
    \begin{talign} \label{eq:w_m_truncation_err}
        &\|w - \mathrm{Pr}\{\tilde t < m\} \cdot w^{[m]}\|_2 \nonumber
        \\ &\leq 
        % \sqrt{\vert \cZ \vert}\|w^{[m]} - w \|_\infty \nonumber
        % \\ &=
        \sqrt{\vert \cZ \vert}\|\kappa_{\sor}^{\tpi} - \mathrm{Pr}\{\tilde t < m\} \cdot \kappa_{\sor,[m]}^{\tpi} \|_\infty  \nonumber
        \\ &=
        \sqrt{\vert \cZ \vert} \max_{s \in \cS}\big\{ (1-\gamma)\sum_{t=m}^\infty \gamma^t P^{(t)}_{\tpi}(\sor, s) \big\}  \nonumber
        \\ &\leq
        \sqrt{\vert \cZ \vert} \cdot \gamma^m.
    \end{talign}
    % Note that $W^{(\infty)} = W = \mathrm{diag}({\chi}^{\pi,\tpi}_{q,f})$.
    Writing $W = \mathrm{diag}(w)$, then using an argument similar to~\cite[Lemma A.3]{li2024stochastic}, we have
    \begin{talign*}
        &\|\mathbb{E}[\hF(\theta) - \hF(\theta')\vert \tau^{\piexplst}(\sor)]- (F(\theta) - F(\theta'))\|_2
        \\ &\leq
        \|\Phi\|_2 \|W - \mathrm{Pr}\{\tilde t < m\} \cdot W^{[m]}\|_2 
        \\ & \cdot 
        \|I- \gamma P^{\pi}\|_2 \|\Phi (\theta-\theta')\|_2
        \\ &\leq
        \|\Phi\|_2^2\sqrt{\vert \cZ \vert} \gamma^m [1 + \gamma \|P^\pi\|_2] \|\theta-\theta'\|_2
        \\ &\leq
        L \vert \cZ \vert^2 \cdot \gamma^m \|\theta-\theta'\|_2,
    \end{talign*}
    where the second inequality is partially by~\eqref{eq:w_m_truncation_err}, and the last inequality is because $P^\pi$ is a row-stochastic matrix with $\vert \cZ \vert$ rows and $\|\Phi\|_2^2 \leq \|\Phi\|_F^2 \leq \vert \cZ \vert \max_{z \in \cZ}\|\phi(z)\|_2^2 = L\vert \cZ \vert/2$.

    The second inequality in the lemma can be shown similarly as above, where the main difference is term $\|\theta - \theta\|_2$ in the derivation above is replaced by $\|\otheta\|_2$, with an upper bound of $\oTheta$ from~\eqref{eq:otheta_bnd}.
\end{proof}

\subsection{Missing proofs from subsection~\ref{sec:anytime_ctd}} \label{sec:missing_anytime_ctd}
\begin{proof}[Proof of Lemma~\ref{lem:ctd_abs_err}]
    Define $\delta_\ell := \hF(\theta_\ell) - F(\theta_\ell)$ and $\Delta_\ell := \langle \delta_\ell, \theta_\ell - \otheta \rangle$ for any $\ell \geq 0$. 
    Then 
    \begin{talign*}
        &(1 + 2\mu\iota_\ell - 2\iota_\ell^2L^2)\|\theta_{\ell+1}-\otheta\|_2^2
        \\ &\leq
        \|\theta_\ell - \otheta\|_2^2 + \iota_\ell^2\|\delta_\ell\|_2^2 - \iota_\ell \Delta_\ell
        % \\ &\leq
        % [1 + 2\iota_\ell^2 C_2 + \iota_\ell C_1 \gamma^m + C_2\iota_\ell^2]\|\theta_\ell - \otheta\|_2^2
        % + \iota_\ell^2 [2\sigma_2^2+1] 
        % \\ &\hspace{10pt} 
        % - [(\iota_\ell C_1 \gamma^m + C_2\iota_\ell^2)\|\theta_\ell - \otheta\|_2^2 + \iota_\ell^2 + \iota_\ell \Delta_\ell]
        \\ &\leq
        [1 + \iota_\ell C_1 \gamma^m + 3\iota_\ell^2 C_2]\|\theta_\ell - \otheta\|_2^2
        + \iota_\ell^2\sigma^2 
        \\ &\hspace{10pt}
        - [(\iota_\ell C_1 \gamma^m + C_2\iota_\ell^2)\|\theta_\ell - \otheta\|_2^2 + \iota_\ell^2 + \iota_\ell \Delta_\ell],
    \end{talign*}
    where the first inequality was shown in the proof for~\cite[Theorem 3.8]{kotsalis2022simple}, the second inequality is with the help of Lemma~\ref{lem:F_bias} and Lemma~\ref{lem:F_lighttail} and recalling $\sigma^2 = 2\sigma_2^2+1$.

    In view of the above inequality, the rest of the proof follows~\cite[Theorem 3.8]{kotsalis2022simple}, except we handle the error term $\iota_\ell \Delta_\ell$ differently because we seek a high probability bound.
    To do so, consider an arbitrary $k \in [N]$, and suppose we have already shown for any $\ell \in [k]$,
    \begin{talign*}
        \Pr\{ \|\theta_\ell - \otheta\|_2^2 > \mathcal{B}(\ell,\delta)  \} \leq \delta/N,
    \end{talign*}
    which can be done using mathematical induction.
    Note that $\mathcal{B}(\ell, \delta) \leq R_{[k]}^2$ for all $\ell \in [N]$.
    For simplicity, we simply write $R_{[k]}$ as $R$ for the remainder of the proof.
    We also assume the complement of the above event holds for all $\ell \in [k]$ for the rest of the proof.

    Now, define a weighting parameter $\alpha_\ell = (\ell + \ell_0+1)(\ell+\ell_0+2)$.
    Because $\Delta_\ell - \mathbb{E}[\Delta_\ell]$ is a martingale difference sequence, then we have with probability $1-\delta/N$,
    \begin{talign*}
        &\sum_{\ell=0}^{k-1} \alpha_\ell \iota_\ell \Delta_\ell
        \\ &=
        \sum_{\ell=0}^{k-1} \alpha_\ell \iota_\ell \mathbb{E}[\Delta_\ell]
        + \sum_{\ell=0}^{k-1} \alpha_\ell \iota_\ell (\Delta_\ell - \mathbb{E}[\Delta_\ell])
        \\ &\geq
        \sum_{\ell=0}^{k-1} -\alpha_\ell \iota_\ell [C_1 \gamma^{m_\ell}\|\theta_\ell-\otheta\|_2^2 + \sigma_1 \gamma^{m_\ell} \|\theta_\ell - \otheta\|_2]
        \\
        &\hspace{10pt}
        - \sqrt{2\log(\frac{4N}{\delta})\sum_{\ell=0}^{k-1} \alpha_\ell^2 \iota^2[\sqrt{C_2}\|\theta_\ell - \otheta\| + \sigma_2]^2}
        % \\ &\geq
        % \sum_{\ell=0}^{k-1} -\alpha_\ell \iota_\ell [(C_1 \gamma^{m_\ell} + C_2 \iota_\ell)\|\theta_\ell-\otheta\|_2^2 + \iota_\ell]
        % - 2\sqrt{\log(\frac{4N}{\delta})\sum_{\ell=0}^{k-1} \alpha_\ell^2 \iota^2[C_2R^2 + \sigma_2^2]},
        \\ &\geq
        \sum_{\ell=0}^{k-1} -\alpha_\ell \iota_\ell [(C_1 \gamma^{m_\ell} + C_2 \iota_\ell)\|\theta_\ell-\otheta\|_2^2 + \iota_\ell]
        \\ &\hspace{10pt}
        - \frac{16}{\mu}\sqrt{\log(\frac{4N}{\delta}) (k + \ell_0)^3 C_2R^2},
    \end{talign*}
    where the first inequality is by Lemma~\ref{lem:F_bias}, the Azuma-Hoeffding inequality, and Lemma~\ref{lem:F_lighttail}. 
    The second inequality is by $2ab \leq a^2 +  b^2$ for any $a,b \in \mathbb{R}$ and $\gamma^{m_\ell} \leq \sqrt{C_2}\iota_\ell/\sigma_1$ since $m_\ell \geq \unm_\ell$, followed by the simplification $\sum_{\ell=0}^{k-1} \alpha_\ell^2 \iota_\ell^2 \leq \frac{4}{\mu^2} \sum_{\ell=0}^{k-1} \frac{16(\ell+\ell_0)^4}{(\ell+\ell_0)^2} \leq \frac{64}{2\mu^2}(k+\ell_0)^3$ and $\sigma_2^2 = C_2 \oTheta^2 \leq R^2$.
    The rest of the proof follows similarly to that of~\cite[Theorem 3.8]{kotsalis2022simple}.
\end{proof}

\begin{proof}[Proof of Lemma~\ref{lem:ctd_bias}]
    Our proof follows closely to~\cite[Lemma 18]{lan2023policy}.
    First, we will derive an absolute bound on the bias of $\hF(\cdot)$. 
    By applying the integral identity~\cite[Lemma 1.2.1]{vershynin2018high} and Lemma~\ref{lem:ctd_abs_err}, we find,
    \begin{talign*}
        (\mathbb{E}\|\theta_\ell - \otheta\|)^2
        \leq
        \mathbb{E}\|\theta_\ell - \otheta\|_2^2 
        \leq 
        \mathcal{B}(\ell,1) 
        \leq 
        R_{[1]}^2 =: R^2,
    \end{talign*}
    where the first inequality is by Jensen's inequality and $\mathcal{B}(\cdot,\cdot)$ is from Lemma~\ref{lem:ctd_abs_err}.
    We can then deduce
    \begin{talign*}
        &\|\mathbb{E}[\hF(\theta_\ell) \vert \tau^{\piexplst}(\sor)] - F(\theta_\ell)\|_2 
        \\ &\leq
        [C_1 \mathbb{E}\|\theta_\ell - \otheta\|_2 + \sigma] \gamma^m 
        \\ &\leq
        [C_1R + \sigma] \gamma^m 
        \leq
        2C_1R \gamma^m,
    \end{talign*}
    where the first inequality is by Lemma~\ref{lem:F_bias}, the second is by the Jensen's inequality, and the third by $\sigma \leq R$ and recalling $C_1 \geq 1$.
    Using the above result and~\eqref{eq:otheta_bnd} within~\cite[Lemma 18]{lan2023policy} guarantees that
    \begin{talign*}
        \|\mathbb{E}\theta_{\ell} - \otheta\|_2^2
        &\leq
        \frac{(\ell_0-1)(\ell_0-2)(\ell_0-3)\|\mathbb{E}\theta_0-\otheta\|_2^2}{(\ell+\ell_0-1)(\ell+\ell_0-2)(\ell+\ell_0-3)}
        \\ &+
        \frac{8C_1R^2\gamma^m}{3\mu} + \frac{2C_1^2R^2\gamma^{2m}}{\mu^2}.
    \end{talign*}
    In particular, with $\ell \geq \ell_0$, we get 
    $\|\mathbb{E}\theta_{\ell} - \otheta\|_2^2 \leq \frac{\|\mathbb{E}\theta_0-\otheta\|_2^2}{8} + \frac{8_1CR^2\gamma^m}{3\mu} + \frac{C_1^2R^2\gamma^{2m}}{\mu^2}$, and we can apply this bound recursively to attain our linear rate of convergence (with respect to $\ell$).
\end{proof}

\subsection{Missing proof from subsection~\ref{sec:auto_linfun}} \label{sec:missing_auto_linfun}

\begin{proof}[Proof of Corollary~\ref{cor:kappa_lb}]
    Let $\unpi := \min_{t} \min_{s,a : \pi^*(a \vert s) =1} \{\tpi_t(a \vert s) \}$. 
    By our choice of $\epspi$-greedy exploration in the exploration policy $\tpi_t$ from~\eqref{eq:tpi_defn}, where we set $\epspi = \frac{1-\gamma}{2}$ in Theorem~\ref{thm:pmd_ctd}, we ensure that $\unpi > \frac{1-\gamma}{2 \vert \cA \vert}$.
    Then we have
    \begin{talign*}
        \kappa^{\tpi_t}_{\sor}(s)
        &\geq
        (1-\gamma) \gamma^{\ob} P^{(\ob)}_{\tpi_t}(\sor,s)
        \\
        &\geq
        (1-\gamma) \frac{(\gamma \unpi)^{\ob} \unu^*}{2}
        \geq
        \frac{(1-\gamma)\unu^*}{2}\big( \frac{(1-\gamma)\gamma}{2\vert \cA \vert}\big)^{\ob},
    \end{talign*}
    where the first inequality is by definition of $\kappa^{\pi}_{\sor}(s)$, the second inequality is by~\eqref{eq:P_pi_t_relation_to_nu} (which uses Assumption~\ref{asmp:optimal_mixing}), and the third is by our bound on $\unpi$.
\end{proof}

\setlength{\bibsep}{1pt}
{\small
% \medskip
% \bibliographystyle{unsrt}
\bibliographystyle{abbrvnat} % <- only shows first initial
\bibliography{biblio}

@article{li2025policy,
  title={Policy mirror descent inherently explores action space},
  author={Li, Yan and Lan, Guanghui},
  journal={SIAM Journal on Optimization},
  volume={35},
  number={1},
  pages={116--156},
  year={2025},
  publisher={SIAM}
}

@article{ju2024strongly,
  title={Strongly-polynomial time and validation analysis of policy gradient methods},
  author={Ju, Caleb and Lan, Guanghui},
  journal={Mathematical Programming},
  pages={1--45},
  year={2026},
  publisher={Springer}
}

@article{lan2023policy,
  title={Policy mirror descent for reinforcement learning: Linear convergence, new sampling complexity, and generalized problem classes},
  author={Lan, Guanghui},
  journal={Mathematical Programming},
  volume={198},
  number={1},
  pages={1059--1106},
  year={2023},
  publisher={Springer}
}

@article{jain2021making,
  title={Making the last iterate of sgd information theoretically optimal},
  author={Jain, Prateek and Nagaraj, Dheeraj M and Netrapalli, Praneeth},
  journal={SIAM Journal on Optimization},
  volume={31},
  number={2},
  pages={1108--1130},
  year={2021},
  publisher={SIAM}
}

@book{levin2017markov,
  title={Markov chains and mixing times},
  author={Levin, David A and Peres, Yuval},
  volume={107},
  year={2017},
  publisher={American Mathematical Soc.}
}

@inproceedings{wolfer2019estimating,
  title={Estimating the mixing time of ergodic markov chains},
  author={Wolfer, Geoffrey and Kontorovich, Aryeh},
  booktitle={Conference on Learning Theory},
  pages={3120--3159},
  year={2019},
  organization={PMLR}
}

@article{hsu2019mixing,
author = {Daniel Hsu and Aryeh Kontorovich and David A. Levin and Yuval Peres and Csaba Szepesv{\'a}ri and Geoffrey Wolfer},
title = {{Mixing time estimation in reversible Markov chains from a single sample path}},
volume = {29},
journal = {The Annals of Applied Probability},
number = {4},
publisher = {Institute of Mathematical Statistics},
pages = {2439 -- 2480},
year = {2019},
}

@article{mou2023optimal,
  title={Optimal oracle inequalities for projected fixed-point equations, with applications to policy evaluation},
  author={Mou, Wenlong and Pananjady, Ashwin and Wainwright, Martin J},
  journal={Mathematics of Operations Research},
  volume={48},
  number={4},
  pages={2308--2336},
  year={2023},
  publisher={INFORMS}
}

@article{li2023accelerated,
  title={Accelerated and instance-optimal policy evaluation with linear function approximation},
  author={Li, Tianjiao and Lan, Guanghui and Pananjady, Ashwin},
  journal={SIAM Journal on Mathematics of Data Science},
  volume={5},
  number={1},
  pages={174--200},
  year={2023},
  publisher={SIAM}
}

@book{sutton1998reinforcement,
  title={Reinforcement learning: An introduction},
  author={Sutton, Richard S and Barto, Andrew G and others},
  volume={1},
  year={1998},
  publisher={MIT press Cambridge}
}

@article{chen2022finite,
  title={Finite-sample analysis of nonlinear stochastic approximation with applications in reinforcement learning},
  author={Chen, Zaiwei and Zhang, Sheng and Doan, Thinh T and Clarke, John-Paul and Maguluri, Siva Theja},
  journal={Automatica},
  volume={146},
  pages={110623},
  year={2022},
  publisher={Elsevier}
}

@article{kotsalis2022simple,
  title={Simple and optimal methods for stochastic variational inequalities, II: Markovian noise and policy evaluation in reinforcement learning},
  author={Kotsalis, Georgios and Lan, Guanghui and Li, Tianjiao},
  journal={SIAM Journal on Optimization},
  volume={32},
  number={2},
  pages={1120--1155},
  year={2022},
  publisher={SIAM}
}

@article{li2024stochastic,
  title={Stochastic first-order methods for average-reward markov decision processes},
  author={Li, Tianjiao and Wu, Feiyang and Lan, Guanghui},
  journal={Mathematics of Operations Research},
  year={2024},
  publisher={INFORMS}
}

@article{li2018hyperband,
  title={Hyperband: A novel bandit-based approach to hyperparameter optimization},
  author={Li, Lisha and Jamieson, Kevin and DeSalvo, Giulia and Rostamizadeh, Afshin and Talwalkar, Ameet},
  journal={Journal of Machine Learning Research},
  volume={18},
  number={185},
  pages={1--52},
  year={2018}
}

@article{agarwal2021theory,
  title={On the theory of policy gradient methods: Optimality, approximation, and distribution shift},
  author={Agarwal, Alekh and Kakade, Sham M and Lee, Jason D and Mahajan, Gaurav},
  journal={Journal of Machine Learning Research},
  volume={22},
  number={98},
  pages={1--76},
  year={2021}
}

@book{puterman2014markov,
  title={Markov decision processes: discrete stochastic dynamic programming},
  author={Puterman, Martin L},
  year={2014},
  publisher={John Wiley \& Sons}
}

@article{rudelson2008invertibility,
  title={Invertibility of random matrices: norm of the inverse},
  author={Rudelson, Mark},
  journal={Annals of Mathematics},
  pages={575--600},
  year={2008},
  publisher={JSTOR}
}

@article{nanda2025minimal,
  title={A minimal-assumption analysis of Q-learning with time-varying policies},
  author={Nanda, Phalguni and Chen, Zaiwei},
  journal={Proceedings of the ACM on Measurement and Analysis of Computing Systems},
  volume={10},
  number={2},
  pages={1--43},
  year={2026},
  publisher={ACM New York, NY, USA}
}

@book{nocedal2006numerical,
  title={Numerical optimization},
  author={Nocedal, Jorge and Wright, Stephen J},
  year={2006},
  publisher={Springer}
}

@article{wang2020statistical,
  title={What are the statistical limits of offline {RL} with linear function approximation?},
  author={Wang, Ruosong and Foster, Dean P and Kakade, Sham M},
  journal={arXiv:2010.11895},
  year={2020}
}

@article{kallus2022efficiently,
  title={Efficiently breaking the curse of horizon in off-policy evaluation with double reinforcement learning},
  author={Kallus, Nathan and Uehara, Masatoshi},
  journal={Operations Research},
  volume={70},
  number={6},
  pages={3282--3302},
  year={2022},
  publisher={INFORMS}
}

@article{gheshlaghi2013minimax,
  title={Minimax {PAC} bounds on the sample complexity of reinforcement learning with a generative model},
  author={Gheshlaghi Azar, Mohammad and Munos, R{\'e}mi and Kappen, Hilbert J},
  journal={Machine Learning},
  volume={91},
  pages={325--349},
  year={2013},
  publisher={Springer}
}

@article{li2024breaking,
  title={Breaking the sample size barrier in model-based reinforcement learning with a generative model},
  author={Li, Gen and Wei, Yuting and Chi, Yuejie and Chen, Yuxin},
  journal={Operations Research},
  volume={72},
  number={1},
  pages={203--221},
  year={2024},
  publisher={INFORMS}
}

@inproceedings{jiang2016doubly,
  title={Doubly robust off-policy value evaluation for reinforcement learning},
  author={Jiang, Nan and Li, Lihong},
  booktitle={International Conference on Machine Learning},
  pages={652--661},
  year={2016},
  organization={PMLR}
}

@article{mnih2015human,
  title={Human-level control through deep reinforcement learning},
  author={Mnih, Volodymyr and Kavukcuoglu, Koray and Silver, David and Rusu, Andrei A and Veness, Joel and Bellemare, Marc G and Graves, Alex and Riedmiller, Martin and Fidjeland, Andreas K and Ostrovski, Georg and others},
  journal={Nature},
  volume={518},
  number={7540},
  pages={529--533},
  year={2015},
  publisher={Nature Publishing Group}
}

@inproceedings{kakade2002approximately,
  title={Approximately optimal approximate reinforcement learning},
  author={Kakade, Sham and Langford, John},
  booktitle={Proceedings of the 19th International Conference on Machine Learning},
  pages={267--274},
  year={2002}
}

@article{tuynman2024finding,
  title={Finding good policies in average-reward {Markov} Decision Processes without prior knowledge},
  author={Tuynman, Adrienne and Degenne, R{\'e}my and Kaufmann, Emilie},
  journal={Advances in Neural Information Processing Systems},
  volume={37},
  pages={109948--109979},
  year={2024}
}

@article{tarbouriech2021provably,
  title={A provably efficient sample collection strategy for reinforcement learning},
  author={Tarbouriech, Jean and Pirotta, Matteo and Valko, Michal and Lazaric, Alessandro},
  journal={Advances in Neural Information Processing Systems},
  volume={34},
  pages={7611--7624},
  year={2021}
}

@article{lee2025near,
  title={Near-optimal sample complexity for {MDP}s via anchoring},
  author={Lee, Jongmin and Bravo, Mario and Cominetti, Roberto},
  journal={arXiv:2502.04477},
  year={2025}
}

@article{even2003learning,
  title={Learning rates for {Q}-learning},
  author={Even-Dar, Eyal and Mansour, Yishay},
  journal={Journal of Machine Learning Research},
  volume={5},
  number={Dec},
  pages={1--25},
  year={2003}
}

@inproceedings{cai2020provably,
  title={Provably efficient exploration in policy optimization},
  author={Cai, Qi and Yang, Zhuoran and Jin, Chi and Wang, Zhaoran},
  booktitle={International Conference on Machine Learning},
  pages={1283--1294},
  year={2020},
  organization={PMLR}
}

@inproceedings{munos2005error,
  title={Error bounds for approximate value iteration},
  author={Munos, R{\'e}mi},
  booktitle={Proceedings of the National Conference on Artificial Intelligence},
  volume={20},
  pages={1006},
  year={2005},
  organization={AAAI}
}

@article{chen2025non,
  title={Non-asymptotic guarantees for average-reward {Q}-learning with adaptive stepsizes},
  author={Chen, Zaiwei},
  journal={arXiv:2504.18743},
  year={2025}
}

@article{ding2023last,
  title={Last-iterate convergent policy gradient primal-dual methods for constrained MDPs},
  author={Ding, Dongsheng and Wei, Chen-Yu and Zhang, Kaiqing and Ribeiro, Alejandro},
  journal={Advances in Neural Information Processing Systems},
  volume={36},
  pages={66138--66200},
  year={2023}
}

@article{cen2022faster,
  title={Faster last-iterate convergence of policy optimization in zero-sum {Markov} games},
  author={Cen, Shicong and Chi, Yuejie and Du, Simon S and Xiao, Lin},
  journal={arXiv:2210.01050},
  year={2022}
}

@article{agarwal2020pc,
  title={{PC-PG}: Policy cover directed exploration for provable policy gradient learning},
  author={Agarwal, Alekh and Henaff, Mikael and Kakade, Sham and Sun, Wen},
  journal={Advances in Neural Information Processing Systems},
  volume={33},
  pages={13399--13412},
  year={2020}
}

@article{jin2018q,
  title={Is {Q}-learning provably efficient?},
  author={Jin, Chi and Allen-Zhu, Zeyuan and Bubeck, Sebastien and Jordan, Michael I},
  journal={Advances in Neural Information Processing Systems},
  volume={31},
  year={2018}
}

@article{ju2022policy,
  title={Policy optimization over general state and action spaces},
  author={Ju, Caleb and Lan, Guanghui},
  journal={arXiv:2211.16715},
  year={2022}
}

@inproceedings{jin2021towards,
  title={Towards tight bounds on the sample complexity of average-reward {MDP}s},
  author={Jin, Yujia and Sidford, Aaron},
  booktitle={International Conference on Machine Learning},
  pages={5055--5064},
  year={2021},
  organization={PMLR}
}

@book{bertsekas1996stochastic,
  title={Stochastic optimal control: the discrete-time case},
  author={Bertsekas, Dimitri and Shreve, Steven E},
  volume={5},
  year={1996},
  publisher={Athena Scientific}
}

@article{khan2020systematic,
  title={A systematic review on reinforcement learning-based robotics within the last decade},
  author={Khan, Md Al-Masrur and Khan, Md Rashed Jaowad and Tooshil, Abul and Sikder, Niloy and Mahmud, MA Parvez and Kouzani, Abbas Z and Nahid, Abdullah-Al},
  journal={IEEE Access},
  volume={8},
  pages={176598--176623},
  year={2020},
  publisher={IEEE}
}

@article{enda13applying,
author = {Barrett, Enda and Howley, Enda and Duggan, Jim},
title = {Applying reinforcement learning towards automating resource allocation and application scalability in the cloud},
journal = {Concurrency and Computation: Practice and Experience},
volume = {25},
number = {12},
pages = {1656-1674},
year = {2013}
}

@article{ouyang2022training,
  title={Training language models to follow instructions with human feedback},
  author={Ouyang, Long and Wu, Jeffrey and Jiang, Xu and Almeida, Diogo and Wainwright, Carroll and Mishkin, Pamela and Zhang, Chong and Agarwal, Sandhini and Slama, Katarina and Ray, Alex and others},
  journal={Advances in Neural Information Processing Systems},
  volume={35},
  pages={27730--27744},
  year={2022}
}

@article{bellman1958dynamic,
  title={Dynamic programming and stochastic control processes},
  author={Bellman, Richard},
  journal={Information and control},
  volume={1},
  number={3},
  pages={228--239},
  year={1958},
  publisher={Elsevier}
}

@inproceedings{alacaoglu2022natural,
  title={A natural actor-critic framework for zero-sum Markov games},
  author={Alacaoglu, Ahmet and Viano, Luca and He, Niao and Cevher, Volkan},
  booktitle={International Conference on Machine Learning},
  pages={307--366},
  year={2022},
  organization={PMLR}
}

@inproceedings{fan2020theoretical,
  title={A theoretical analysis of deep {Q}-learning},
  author={Fan, Jianqing and Wang, Zhaoran and Xie, Yuchen and Yang, Zhuoran},
  booktitle={Learning for Dynamics and Control},
  pages={486--489},
  year={2020},
  organization={PMLR}
}

@article{zhu2024uncertainty,
  title={Uncertainty quantification and exploration for reinforcement learning},
  author={Zhu, Yi and Dong, Jing and Lam, Henry},
  journal={Operations Research},
  volume={72},
  number={4},
  pages={1689--1709},
  year={2024},
  publisher={INFORMS}
}

@book{vershynin2018high,
  title={High-dimensional probability: An introduction with applications in data science},
  author={Vershynin, Roman},
  volume={47},
  year={2018},
  publisher={Cambridge university press}
}

@article{dann2014policy,
  title={Policy evaluation with temporal differences: A survey and comparison},
  author={Dann, Christoph and Neumann, Gerhard and Peters, Jan},
  journal={The Journal of Machine Learning Research},
  volume={15},
  number={1},
  pages={809--883},
  year={2014},
  publisher={JMLR. org}
}

@book{shapiro2021lectures,
  title={Lectures on stochastic programming: modeling and theory},
  author={Shapiro, Alexander and Dentcheva, Darinka and Ruszczynski, Andrzej},
  year={2021},
  publisher={SIAM}
}

@article{michie1968boxes,
  title={{BOXES}: An experiment in adaptive control},
  author={Michie, Donald and Chambers, R.A.},
  journal={Machine intelligence},
  volume={2},
  pages={137--152},
  year={1968},
  publisher={Oliver \& Boyd}
}
}

\end{document}